\documentclass[11pt]{article}

\PassOptionsToPackage{table}{xcolor}
\usepackage[preprint]{acl}

\usepackage{times}
\usepackage{latexsym}
\usepackage[T1]{fontenc}
\usepackage[utf8]{inputenc}
\usepackage{microtype}
\usepackage{inconsolata}

\usepackage{graphicx}
\usepackage{subcaption}
\usepackage{booktabs}
\usepackage{amsfonts, amsmath, amssymb, amsthm}
\usepackage{mathtools}
\usepackage{ulem}
\usepackage{float}
\usepackage{algorithmic}
\usepackage[ruled,vlined]{algorithm2e}
\usepackage{color,soul}
\usepackage{multirow}
\usepackage[capitalize,noabbrev]{cleveref}
\usepackage[textsize=tiny]{todonotes}
\usepackage[most]{tcolorbox}
\usepackage{tabularx}
\usepackage{wrapfig}

\usepackage{amsmath,amsfonts,bm}

\def\eqref#1{equation~\ref{#1}}

\def\1{\bm{1}}

\DeclareMathAlphabet{\mathsfit}{\encodingdefault}{\sfdefault}{m}{sl}
\SetMathAlphabet{\mathsfit}{bold}{\encodingdefault}{\sfdefault}{bx}{n}

\newcommand{\E}{\mathbb{E}}

\newcommand{\nc}{\newcommand}
\nc{\dc}{\definecolor}
\dc{watpink}{RGB}{198, 0, 120}
\dc{editgreen}{RGB}{90, 150, 90}
\nc{\wpink}[1]{\textcolor{watpink}{#1}}
\nc{\yck}[1]{\wpink{[YCK: #1]}}

\nc{\const}{\mathrm{const}}

\newcommand{\meanstd}[2]{#1\,{\scriptsize$\pm$\,#2}}

\newcommand{\ph}[1]{\underline{\texttt{<#1>}}}

\newcolumntype{Y}{>{\centering\arraybackslash}X}
\definecolor{darkred}{rgb}{0.55, 0.0, 0.0}
\definecolor{baselinegray}{gray}{0.95}
\definecolor{methodgreen}{RGB}{235,245,235}
\definecolor{methodblue}{RGB}{235,242,252}

\theoremstyle{plain}
\newtheorem{theorem}{Theorem}[section]
\newtheorem{proposition}[theorem]{Proposition}

\theoremstyle{definition}

\theoremstyle{remark}

\title{Patients-like-me: A Variational LM--GNN Framework for Explainable Clinical Prediction}

\author{%
  \textnormal{Xinyu Wang}$^{1}$\thanks{Equal contribution. Emails:
  \texttt{xinyu.wang5@mail.mcgill.ca} and
  \texttt{yixuan.li2@mail.mcgill.ca}.}
  \quad
  \textnormal{Yixuan Li}$^{1}$\footnotemark[1]
  \quad
  \textnormal{Hanwei Wu}$^{2}$
  \quad
  \textnormal{Qincheng Lu}$^{1}$
  \\
  \textnormal{Chi-Kuang Yeh}$^{3}$
  \quad
  \textnormal{Xiao-Wen Chang}$^{1}$
  \quad
  \textnormal{Ziyang Song}$^{4}$\thanks{Corresponding author:
  \texttt{ziyangs@ohio.edu}.}
  \\
  \small
  $^{1}$McGill University
  \quad
  $^{2}$Université de Montréal
  \quad
  $^{3}$Georgia State University
  \quad
  $^{4}$Ohio University
}

\begin{document}
\raggedbottom

\maketitle

\begin{abstract}
Language models (LMs) offer strong textual representations for electronic health records (EHRs), but they encode patient sequences in isolation and provide limited explainability. Graph neural networks (GNNs) complement LMs by incorporating inter-patient relationships and enabling reference-patient attribution, yet they rely on high-quality patient representations. We propose \textbf{Patients-like-me (PLM)}, a unified LM–GNN framework that integrates local patient semantics with global cohort structure. To train PLM efficiently, we introduce a \textbf{Variational Expectation-Maximization} algorithm that alternates LM and GNN updates under a supervised variational objective. Extensive experiments on MIMIC-III and MIMIC-IV show that PLM consistently outperforms state-of-the-art methods, with improvements generalizing across encoder-only and decoder-only LM backbones. These gains are achieved with only modest additional computational overhead. PLM also provides reference-patient explanations by retrieving influential similar patients, while edge-masking experiments confirm that the highest-ranked references have the greatest impact on model predictions.
\end{abstract}

\section{Introduction}
\label{sec:intro}

Electronic health records (EHRs) contain heterogeneous patient records spanning one or more visits, including structured codes (e.g., diagnoses, procedures, medications) and unstructured free-text notes \citep{xu2022survey}. The accumulation of large-scale EHR data enables representation learning that supports accurate \emph{supervised} clinical prediction for clinical decision support, including readmission and length-of-stay (LOS) prediction as well as personalized medication recommendation \citep{ijcai2024p914,shang2019gamenet,yang2021safedrug}.

Language models (LMs) are a natural backbone for EHR modeling because they can encode structured codes and free-text notes into textual representations \citep{wornow2025contextclues}. Recent work therefore leverages LMs as patient-level encoders, mapping each patient's EHR sequence to a representation for downstream predictions \citep{makarov2025dtgpt,hegselmann2025llmsehr}. However, LM-based encoders are inherently \emph{local}, representing each patient sequence in isolation, thus overlooking \emph{global} cohort structure such as shared phenotypes and comorbidities. This hurts tasks that benefit from cohort-level signals, such as personalized drug recommendation, where clinically similar reference patients (i.e., ``patients like me'') provide evidence for prediction \citep{Alsentzer2025-em,Kauffman2025-zf}. For example, among patients with congestive heart failure, medication selection is often influenced by shared comorbidities such as kidney dysfunction or diabetes. In addition, LMs offer limited explainability, making it difficult to provide explanations linked to reference cases.

Graph neural networks (GNNs) complement LMs by modeling inter-patient relations on a text-attributed graph (TAG), where nodes represent patient sequences and edges reflect cohort similarity \citep{Yan2023CSTAG,Boll2024GNNsurvey}. These structural relationships preserve patient proximity, such that connected patients are more likely to exhibit similar clinical profiles. GNN message passing propagates cohort-level signals to refine patient representations, addressing the locality limitation of LMs. In addition, graph-based attribution methods can link predictions to influential neighboring cases, yielding reference-patient explanations for prediction \citep{ying2019gnnexplainer}. This motivates a unified LM–GNN framework in which LMs encode patient representations and GNNs refine them using global cohort structure.

Existing LM-GNN methods mainly adopt two-stage, end-to-end, or alternating optimization. Two-stage methods train the LM and GNN separately, offering scalability but preventing graph information from iteratively refining the LM~\citep{lm_graph_survey}. End-to-end methods jointly optimize both components, enabling tighter interaction but incurring higher memory and computational costs~\citep{huang2024can,hu2025large}. Alternating methods, such as GLEM~\citep{zhao2022glem}, update the LM and GNN in turn to balance interaction and efficiency, but are primarily developed for semi-supervised node classification with missing labels. Their formulation therefore does not address fully supervised patient-level clinical prediction.

In this work, we propose \emph{Patients-like-me} (PLM), a unified LM-GNN framework for clinical prediction. PLM formulates a \emph{Variational Expectation-Maximization} (VEM) algorithm that   alternates LM-GNN optimization under a supervised variational objective, allowing local patient semantics and global cohort structure to iteratively refine one another with modest computational overhead. Experiments on MIMIC-III and MIMIC-IV across readmission prediction, length-of-stay prediction, and drug recommendation show that PLM consistently outperforms existing clinical prediction models, LM-based baselines, and alternative LM-GNN training strategies. PLM is backbone-agnostic, improving both encoder-only and decoder-only models with only modest computational overhead. Beyond predictive performance, PLM provides reference-patient explanations that identify influential similar patients, with edge-masking experiments showing that the highest-ranked references have the strongest effect on predictions.

\section{Related Work}
\label{sec:related}

\subsection{Language Modeling for EHRs}

Modeling EHR sequences is central to clinical prediction.
EHRs consist of heterogeneous patient records, where each visit may contain structured clinical codes and free-text notes~\citep{xu2022survey}. Early probabilistic models are employed to uncover latent structure from EHRs~\citep{mixehr-seed}. Subsequent work uses deep neural networks to learn representations from patient records, including multi-layer perceptrons (MLPs) ~\citep{Miotto2016-iq}, recurrent networks~\citep{choi2016doctorai, pham2016deepcare}, convolutional networks~\citep{Ma2020-gf}, and attention-based models~\citep{RETAIN, GRAM}. Recently, Transformer-based LMs have become the dominant paradigm for encoding EHR data into contextual patient representations~\citep{lee2020biobert, huang2019clinicalbert, rasmy2021med, labrak2024biomistral}. These models span both  encoder-only architectures (e.g., BioBERT \citep{lee2020biobert}, BioClinical ModernBERT \citep{BCMBERT}) and decoder-only LMs (e.g., BioMistral, Meerkat) \citep{labrak2024biomistral, Meerkat}. However, these biomedical LMs primarily encode local textual representations~\citep{defilippo2024triage}, motivating the incorporation of inter-patient structure to complement these representations~\citep{Alsentzer2025-em, Kauffman2025-zf}.
\begin{figure*}[!t]
\centering
  \includegraphics[width=0.9\linewidth]{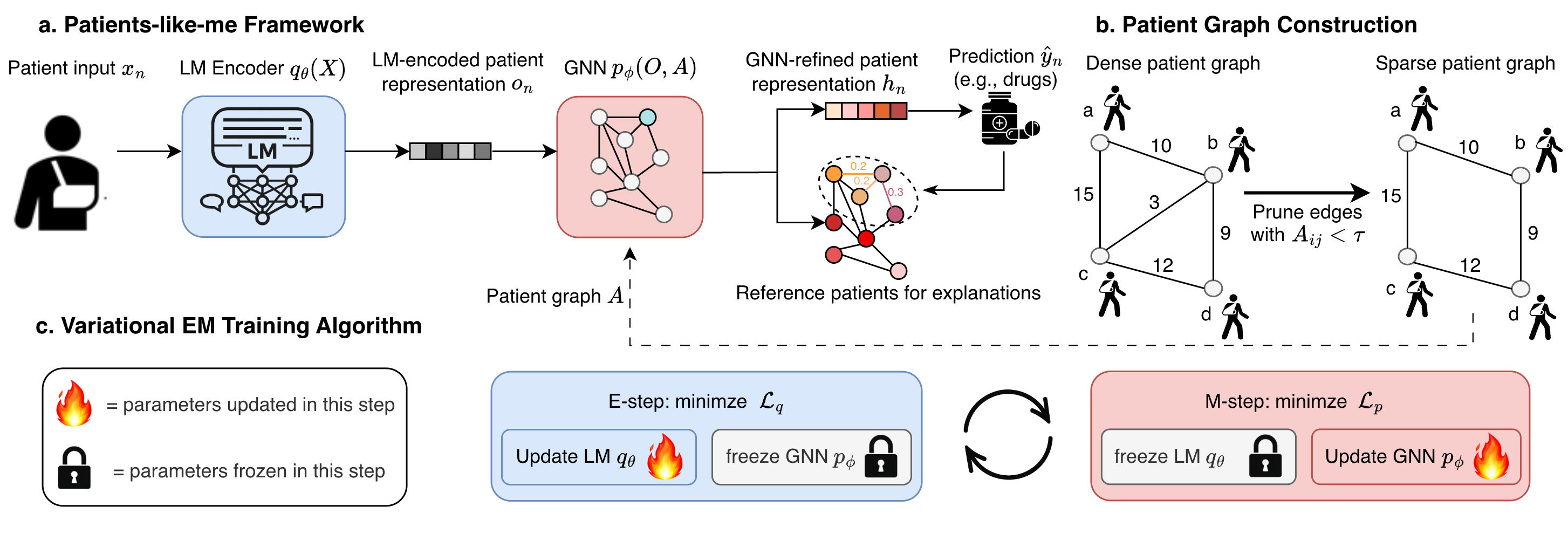}
  \caption{\textbf{Overview of Patients-like-me.} \textbf{(a)} The LM encodes each patient sequence into an embedding, which the GNN refines using patient graph \(A\) for prediction and reference patient explanations.  \textbf{(b)} The patient graph \(A\) retains edges with \(A_{ij}\ge\tau\). \textbf{(c)} The variational EM algorithm alternates between an E-step that updates the LM while freezing the GNN and an M-step that updates the GNN while freezing the LM.}
  \label{fig:outline}
  \vspace{-1\baselineskip} 
\end{figure*}

\subsection{Graph-Based Patient Modeling and LMs}
GNNs on patient graphs are widely used to exploit cohort structure for clinical prediction~\citep{Boll2024GNNsurvey,vaida2025multimodal}. Applying GNNs to EHRs requires mapping heterogeneous patient records to continuous node features \citep{gilmer2017neural}, motivating TAGs and textual representation learning for node initialization \citep{Yan2023CSTAG}. Prior work computes node embeddings from hand-crafted features or shallow text features (e.g., bag-of-words)~\citep{DBLP:conf/iclr/KipfW17}. Sequential models have been adopted to map patient sequences into node embeddings for GNNs~\citep{liu2021learning,rocheteau2021predicting,dong2023integrated}. Recent work uses biomedical LMs to produce semantically rich embeddings for GNNs~\citep{Chen2024IJCAI, graphcare2024}.
Most LM–GNN pipelines rely on two-stage or end-to-end optimization methods, but the former is often ineffective and the latter is not scalable to large data~\citep{golmaei2021deepnote,hu2025large}. Alternating methods iteratively update LMs and GNNs, but are tailored to semi-supervised node classification with partial labels ~\citep{zhao2022glem}. Beyond training, prior GNN explainability methods such as GNNExplainer~\citep{ying2019gnnexplainer}, PGExplainer~\citep{luo2020pgexplainer}, and SubgraphX~\citep{yuan2021subgraphx} provide edge- or subgraph-level attributions. In PLM, we leverage graph-based neighbor attribution to provide reference-patient explanations for clinical prediction.

\section{Methodology}


\subsection{Patients-like-me Architecture}
\label{subsec:Patients-Like-Me}

For clinical prediction, each patient $n$ is associated with an EHR sequence $x_n$ and a target label $y_n$. PLM couples a patient-level LM with a cohort-level GNN within a VEM framework (Fig.~\ref{fig:outline}a). We first serialize each EHR into a text prompt by concatenating demographics with chronologically ordered visits, where each visit includes co-occurring diagnosis, procedure, and medication codes.
The LM encodes the prompt into a patient embedding $o_n$, which initializes the GNN node feature as $h_n^{(0)} = o_n$. The GNN then propagates information over the patient graph to refine node representations $h_n^{(L)}$ for downstream prediction. Task-specific prompt templates are in Appendix~\ref{app:prompt}.

As illustrated in Fig.~\ref{fig:outline}b, we construct a weighted patient graph from the cohort. We represent it as a TAG $\mathcal{G}=(\mathcal{N}, A, x)$, where $\mathcal{N}$ is the set of patient nodes, $A \in \mathbb{R}^{N\times N}$ is the adjacency matrix, and each node is associated with a patient EHR sequence. We construct the patient-patient graph using diagnosis-code overlap, following prior work~\citep{Lu2021-re}. To prevent temporal leakage, edges for each prediction instance are computed only from diagnosis codes available within the corresponding input prefix.

Specifically, the edge weight $A_{ij}$ is defined as the number of diagnosis codes shared by patients $i$ and $j$. This graph captures coarse inter-patient comorbidity structure, while patient-specific details are encoded by the LM. To maintain sparsity and suppress noisy links, we retain an edge only if $A_{ij} \geq \tau$. We select $\tau$ on the validation set and use it across all tasks, as it provides the best trade-off between predictive performance and graph sparsity (Section~\ref{subsec:sensitivity}). We model the resulting weighted patient graph with a Graph Convolutional Network (GCN)~\citep{kipf2017semi}, which supports weighted message passing under standard GCN normalization. After the patient-level split, we construct separate training, validation, and test graphs within each split. Details on split-specific graph construction, test-time message passing, and edge-weight analysis are provided in Appendix~\ref{app:weighted_graph}.



\subsection{VEM Algorithm} 
\label{sec:VI}

As illustrated in Fig.~\ref{fig:outline}c, we present a \textit{pseudo-likelihood VEM algorithm} for PLM. We consider a probabilistic model in which patient sequences $x$ are observations and target labels $y$ are latent variables. Our goal is to infer the posterior distribution $p(y \mid x)$, which is generally intractable \citep{Fox2012-hy}. Variational inference approximates this posterior with a tractable family $q_\theta(y \mid x)$ and optimizes the evidence lower bound (ELBO) \citep{VI}:
$
    \mathrm{ELBO} = \E_{q_\theta(y \mid x)} \big[\log p_\phi (x, y) - \log q_\theta(y \mid x ) \big],$
where $\theta$ and $\phi$ denote the variational and model parameters, respectively. 
Following amortized variational inference \citep{vae}, we approximate the intractable posterior with a variational distribution $q_\theta(y \mid x)$ parameterized by an LM, which maps each patient sequence $x_n$ to a distribution over $y_n$ by a single forward pass. We use a mean-field variational family that factorizes over the latent variables,
\(
q_\theta(y \mid x) = \prod_{n=1}^{N} q_\theta(y_n \mid x_n),
\)
where the factorization is imposed only on the variational family $q_\theta$ for scalable amortized inference. We instantiate the distribution $p_\phi(y \mid x)$ with a GNN and omit the explicit dependence on the patient graph $A$.
We optimize the ELBO using a VEM algorithm that alternates between E-step and M-step.
In the E-step, we fix the GNN and update the LM by minimizing the KL divergence
\(
    \mathrm{KL}\bigl(q_\theta(y \mid x)\,\|\,p_\phi(y \mid x)\bigr)
=
\mathbb{E}_{q_\theta(y \mid x)}\!\left[\log q_\theta(y \mid x) - \log p_\phi(y \mid x)\right].
\)
This injects graph structural signals into the LM representations.  In the M-step, we fix the LM and update the GNN by maximizing a pseudo-likelihood objective~\citep{Besag1975-zl},
\(
\mathbb{E}_{q_\theta(y \mid x)}\big[\log p_\phi(y \mid x)\big]
\approx
\mathbb{E}_{q_\theta(y \mid x)}\left[\sum_{n=1}^{N} \log p_\phi\bigl(y_n \mid x, y_{\mathcal{N} \setminus n}\bigr)\right],
\)
where \(y_{\mathcal{N} \setminus n}\) denotes the labels of all nodes except node \(n\). It improves the GNN’s modeling of global structural signals using the LM representations.

\subsection{E-Step: LM Optimization}\label{subsec:e-step}

In the E-step, we fix the GNN and update the LM by maximizing the ELBO, thereby injecting global relational structure into the LM. Under the mean-field factorization, the optimal update for the factor of $q_\theta(y_n\mid x_n)$ for patient $n$ is given by the following theorem.

\begin{theorem}[Optimal mean-field update of $q_\theta$] \label{thm:q_opt}
Given the fixed $p_\phi(y_n \mid x_n)$, the local optimum of the mean-field variational distribution $q_\theta(y_n \mid x_n)$, denoted by $q_\theta^*(y_n \mid x_n)$, satisfies
\begin{align} \label{eq:q_optimum_main}
    \log q_\theta^*(y_n \mid x_n) &= \E_{q_\theta(y \mid x)} \big[ \log p_\phi(y_n\mid x_n, y_{\mathrm{NB}(n)}) \big] \nonumber 
    \\
    &\quad  + \const,
\end{align}
where the constant term is independent of $y_n$ and can be omitted during optimization.
\end{theorem}
The proof is provided in Appendix~\ref{app: optimum}. Based on Theorem~\ref{thm:q_opt}, computing the local optimum $q_\theta^*(y_n \mid x_n)$ requires evaluating an expectation under the current variational distribution $q_\theta(y \mid x)$. We estimate this expectation with a single-sample Monte Carlo estimator~\citep{vae}. As shown in Appendix~\ref{app:sampling_analy}, a single sample yields performance comparable to that of multiple samples, so we use a single sample in all experiments for efficiency.
This Monte Carlo approximation is implemented using a single forward pass of the LM to obtain the patient embedding \(o_n\) for each node, which is then used as the input node feature for the GNN, i.e., \(h_n^{(0)} = o_n\), yielding
\begin{equation} \label{eq:expect_q_p}
\begin{split}
    &\mathbb{E}_{q_\theta(y \mid x)} \big[ \log p_\phi(y_n \mid x_n, y_{\mathrm{NB}(n)}) \big] \\
    &\quad \simeq \log p_\phi(y_n \mid o_n, y_{\mathrm{NB}(n)}).
\end{split}
\end{equation}
Given \eqref{eq:q_optimum_main} and \eqref{eq:expect_q_p}, the local optimum $q_\theta^{*}(y_n \mid x_n)$ can thus be approximated by $p_\phi(y_n \mid o_n, y_{\mathrm{NB}(n)})$. This optimality condition indicates that, during the E-step, the LM is optimized through the fixed GNN, allowing its patient representations to incorporate both local semantics and neighborhood interactions.

The E-step seeks to maximize the ELBO, equivalently minimizing $\mathrm{KL} \big[q_\theta(y \mid x) || p_\phi(y \mid x)\big]$. However, directly optimizing this KL is typically intractable as it requires computing over all possible $q_\theta(y \mid x)$. To address this issue, we adopt the wake-sleep algorithm, which updates the inference model \(q_\theta\) by minimizing a reverse KL surrogate~\citep{Hinton1995-wi},
\begin{equation} \label{eq:sleep_objective_full}
\begin{split}
    \theta^*
    &= \arg \min_\theta \mathrm{KL}\bigl(p_\phi(y \mid x)\,\|\,q_\theta(y \mid x)\bigr) \\
    &= \arg \max_\theta \mathbb{E}_{p_\phi(x, y)} \big[\log q_\theta(y \mid x)\big].
\end{split}
\end{equation}
We omit the constant term independent of \(\theta\). Combining the above results, we obtain the supervised E-step objective for updating \(q_\theta\) by fine-tuning the LM for label prediction,
\begin{equation}\label{eq:e-step-obj}
\begin{split}
    \mathcal{L}_q
    &= \mathbb{E}_{p_\phi(x,y)} \big[\log q_\theta(y \mid x)\big] \\
    &= \mathbb{E}_{p_\phi(x,y)} \left[\sum_{n=1}^{N} \log q_\theta\!\left(y_n \mid o_n\right)\right] \\
    &\approx \sum_{n=1}^{N} \log q_\theta\!\left(y_n \mid h_n^{(L)}(o; \bar{\phi})\right),
\end{split}
\end{equation}
where \(h_n^{(L)}(o; \bar{\phi})\) denotes the deterministic GNN representation computed from the LM embeddings \(o\) under the frozen GNN parameters \(\bar{\phi}\). Since expectation over $p_\phi(x,y)$ is intractable, we estimate it by a single-sample Monte Carlo sampling. In the E-step, the LM first generates embeddings $o_n$, which are passed through the frozen GNN to obtain $h_n^{(L)}(o,\bar{\phi})$. Because the GNN parameters are fixed, the loss is computed on $h_n^{(L)}(o,\bar{\phi})$ in place of $o_n$. Gradients backpropagate through the fixed GNN to $o_n$, updating only the LM parameters $\theta$.

\subsection{M-Step: GNN Optimization}\label{subsec:m-step}
Note that $\log p_\phi(x, y)=\log p_\phi(y \mid x) + \log p_\phi(x)$. 
In the M-step, we fix the LM and update the GNN to maximize the expected log-likelihood 
$   
    \E_{q_\theta(y \mid x)} \big[\log p_\phi(x, y) \big] 
     %
     = \E_{q_\theta(y \mid x)} \big[\log p_\phi(y \mid x) \big]  + \log p_\phi(x),
$
where the second term on the right hand side is independent of $\phi$ since the GNN only parameterizes $p_\phi(y\mid x)$, and is therefore omitted when optimizing  $\phi$. Directly maximizing $\log p_\phi(y\mid x)$ is often intractable because the partition function requires summing over the exponentially large label space to normalize the distribution \citep{Wainwright2008-jv}. Following the pseudo-likelihood approximation~\citep{Besag1975-zl}, we instead maximize a product of tractable local conditional likelihoods, where each term $y_n$ is conditioned on the remaining labels $y_{\mathcal{N} \setminus n}$ and the full observed evidence $x$:
$
    \E_{q_\theta(y \mid x)} \left[\sum_{n}^N \log p_\phi(y_n \mid x, y_{\mathcal{N} \setminus n}) \right] = \E_{q_\theta(y \mid x)} \left[\sum_{n}^N \log p_\phi(y_n \mid x_n, y_{\mathrm{NB}(n)}) \right],   
$
where the graphical Markov property implies that the local conditional depends on neighboring labels $y_{\mathrm{NB}(n)}$.
In the implementation, neighborhood dependencies are modeled through message passing over neighboring representations rather than conditioning on neighboring labels, since the latter are supervision targets and are not used as inputs.
We estimate the expectation under $q_\theta(y \mid x)$ using a single Monte Carlo sample, and represent the node $x_n$ by its LM embedding $o_n$. The M-step updates the GNN by maximizing the log-likelihood, equivalently minimizing the CE loss:
\begin{equation}\label{eq:m-step-obj}
    \mathcal{L}_p = \sum_{n}^N \log p_\phi(y_n \mid o_n, y_{\mathrm{NB}(n)}).
\end{equation}


\subsection{Optimization}
\label{subsec:optiz}

Algorithm~\ref{alg:optim} summarizes our VEM algorithm, which alternates between updating the LM $q_\theta$ and the GNN $p_\phi$ while fixing the other component. In the E-step, we freeze the GNN parameters $\phi$ and update the LM by minimizing the loss in \eqref{eq:e-step-obj} through the GNN output $h_n^{(L)}$. Gradients update only the LM parameters $\theta$. In the M-step, we freeze the LM parameters $\theta$, use the LM representations $o_n$ as fixed node features, and update the GNN by minimizing the loss in \eqref{eq:m-step-obj}. Gradients update only the GNN parameters $\phi$. 
During inference, the LM first encodes $x_n$ into $o_n$ and initializes the node feature as $h_n^{(0)} = o_n$, after which the GNN computes $h_n^{(L)}$ for prediction. At validation and test time, both modules are fixed and applied to the corresponding cohort graph for prediction.

\subsection{Reference-Patient Attribution}
\label{subsec:interpret_analysis}
PLM provides reference-patient explanations by retrieving \emph{reference patients} for a target patient $i$ from the patient graph $A$. For each target patient, we retrieve up to 10 reference patients. When more than 10 neighbors are available, we select the top 10 based on cosine similarity of their embeddings to focus on more similar patients. For each reference patient $j$, we quantify its contribution to the prediction for patient $i$ via a gradient-based importance score, $\mathrm{Imp}(i,j)= \partial \, \mathrm{logit}_i / \partial \, A_{ij}$~\citep{saliency}. This is defined as the gradient of the target logit with respect to the edge weight $A_{ij}$, which reflects the contribution of $j$. We rank reference patients using their signed importance scores $\mathrm{Imp}(i,j)$ and report the top peers as explanations for prediction.
To quantitatively assess the faithfulness of our reference-patient attribution, we perform a masking-based evaluation by removing selected reference edges and measuring the resulting prediction drop.

\section{Experiments}
\label{sec:exp}

\subsection{Experimental Setting}
\label{sec:setup}

\paragraph{Datasets and Preprocessing.}
We use two large-scale, de-identified EHR datasets, MIMIC-III and MIMIC-IV~\citep{johnson2016mimic,johnson2023mimiciv}. Each patient is represented as a temporally ordered visit sequence, with diagnoses, procedures, and medications treated as unordered sets within each visit. Data statistics and preprocessing are deferred to Appendix~\ref{app:dataset_stats} and \ref{app:preprocess}.

\paragraph{Tasks.} We evaluate three tasks: readmission within 15 days (AUPRC/AUROC), length-of-stay prediction over 10 classes (AUPRC/F1), and drug recommendation (AUPRC/F1/Jaccard). We use BCE for readmission and drug recommendation, and CE for LOS. Detailed task settings and metric definitions are provided in Appendix~\ref{app:tasks}.

\paragraph{Baselines.}
We compare against representative clinical prediction models, namely Deepr~\citep{nguyen2017deepr}, RETAIN~\citep{RETAIN}, GRAM~\citep{GRAM}, StageNet~\citep{gao2020stagenet}, AdaCare~\citep{Ma2020-gf}, and GRASP~\citep{zhang2021grasp}. We further compare with LM-graph models that incorporate external medical knowledge or graph, including G-BERT~\citep{shang2019gbert}, LEADER~\citep{liu2024leader}, GLEM~\citep{zhao2022glem}, GraphCare~\citep{graphcare2024}, KARE~\citep{KARE}, ColaCare~\citep{ColaCare}. For drug recommendation, we also include SafeDrug~\citep{yang2021safedrug}, MICRON~\citep{yang2021micron},
GAMENet~\citep{shang2019gamenet}, MoleRec~\citep{yang2023molerec}, UDC~\citep{UDC}, and KEHGCN~\citep{Zhang2026-li}. Details about baseline implementations are provided in Appendix~\ref{app:baselines}.

\begin{table*}[!t]
\centering
\caption{\textbf{Clinical prediction performance on MIMIC-III and IV.}
We report the mean (standard deviation) performance (\%) over 10 runs. The best results are highlighted for both datasets, and the second-best results are \protect\underline{underlined}. N/A indicates not applicable.}
\label{tab:results_mimic34}
\scriptsize
\setlength{\tabcolsep}{5pt}
\renewcommand{\arraystretch}{1.1}
\resizebox{\textwidth}{!}{%
\begin{tabularx}{\textwidth}{@{}>{\raggedright\arraybackslash}p{3.2cm}YYYYYYYY@{}}
\toprule
\multirow{3}{*}{\textbf{Model (\%)}} &
\multicolumn{4}{c}{\textbf{Task 1: Readmission}} &
\multicolumn{4}{c}{\textbf{Task 2: LOS}} \\
\cmidrule(lr){2-5}\cmidrule(lr){6-9}
& 
\multicolumn{2}{c}{\textbf{MIMIC-III}} & \multicolumn{2}{c}{\textbf{MIMIC-IV}} &
\multicolumn{2}{c}{\textbf{MIMIC-III}} & \multicolumn{2}{c}{\textbf{MIMIC-IV}} \\
\cmidrule(lr){2-3}\cmidrule(lr){4-5}\cmidrule(lr){6-7}\cmidrule(lr){8-9}
& \textbf{AUPRC} & \textbf{AUROC} & \textbf{AUPRC} & \textbf{AUROC}
& \textbf{AUPRC} & \textbf{F1} & \textbf{AUPRC} & \textbf{F1} \\
\midrule



\rowcolor{methodgreen}
PLM (BCMBERT-396M)
& \underline{\meanstd{51.7}{0.5}} & \underline{\meanstd{81.2}{0.4}} & \underline{\meanstd{48.6}{0.4}} & \textbf{\meanstd{80.6}{0.4}}
& \underline{\meanstd{83.6}{0.2}} & \underline{\meanstd{37.8}{0.4}} & \underline{\meanstd{82.8}{0.2}} & \underline{\meanstd{35.0}{0.3}} \\

\rowcolor{methodgreen}
PLM (Meerkat-8B)
& \textbf{\meanstd{52.2}{0.4}} & \textbf{\meanstd{81.4}{0.4}} & \textbf{\meanstd{49.7}{0.2}} & \underline{\meanstd{80.3}{0.3}}
& \textbf{\meanstd{84.5}{0.3}} & \textbf{\meanstd{38.6}{0.2}} & \textbf{\meanstd{85.3}{0.1}} & \textbf{\meanstd{35.6}{0.2}} \\
\midrule

Deepr     
& \meanstd{44.8}{0.9} & \meanstd{76.1}{0.4} & \meanstd{42.6}{0.2} & \meanstd{74.7}{0.3}
& \meanstd{77.9}{0.1} & \meanstd{35.0}{0.4} & \meanstd{79.5}{0.3} & \meanstd{32.3}{0.1} \\

RETAIN    
& \meanstd{40.6}{1.0} & \meanstd{71.0}{0.7} & \meanstd{38.5}{0.4} & \meanstd{69.5}{0.5}
& \meanstd{78.2}{0.1} & \meanstd{34.9}{0.4} & \meanstd{78.9}{0.3} & \meanstd{32.0}{0.2} \\

GRAM      
& \meanstd{43.8}{0.7} & \meanstd{74.1}{0.4} & \meanstd{41.4}{0.2} & \meanstd{73.0}{0.3}
& \meanstd{78.2}{0.1} & \meanstd{34.5}{0.2} & \meanstd{78.8}{0.2} & \meanstd{31.9}{0.3} \\

StageNet  
& \meanstd{42.5}{0.6} & \meanstd{73.8}{0.4} & \meanstd{40.2}{0.1} & \meanstd{72.6}{0.1}
& \meanstd{78.3}{0.2} & \meanstd{34.4}{0.4} & \meanstd{79.2}{0.3} & \meanstd{31.3}{0.3} \\

AdaCare   
& \meanstd{44.0}{0.6} & \meanstd{74.1}{0.3} & \meanstd{41.6}{0.1} & \meanstd{72.9}{0.1}
& N/A & N/A & N/A & N/A \\

GRASP     
& \meanstd{43.1}{0.4} & \meanstd{73.4}{0.6} & \meanstd{40.9}{0.3} & \meanstd{72.3}{0.2}
& N/A & N/A & N/A & N/A \\

G-BERT      
& \meanstd{49.6}{0.6}  & \meanstd{79.2}{0.8} & \meanstd{46.3}{0.4} & \meanstd{77.1}{0.7}
& \meanstd{79.7}{0.4} & \meanstd{35.2}{0.3} & \meanstd{78.6}{0.2} & \meanstd{32.0}{0.1} \\

LEADER      
& \meanstd{45.8}{0.5}  & \meanstd{76.3}{0.4} & \meanstd{44.7}{0.2} & \meanstd{75.3}{0.2}
& \meanstd{79.3}{0.3} & \meanstd{35.4}{0.2} & \meanstd{78.2}{0.2} & \meanstd{31.5}{0.2} \\

GLEM      
& \meanstd{45.7}{0.8} & \meanstd{75.1}{0.9} & \meanstd{42.6}{0.6} & \meanstd{73.9}{0.7}
& \meanstd{78.5}{0.3} & \meanstd{35.2}{0.4} & \meanstd{79.0}{0.2} & \meanstd{32.5}{0.3} \\

GraphCare 
& \meanstd{48.8}{0.5} & \meanstd{77.6}{0.6} & \meanstd{45.0}{0.4} & \meanstd{76.3}{0.5}
& \meanstd{81.4}{0.3} & \meanstd{37.5}{0.5} & \meanstd{81.7}{0.2} & \meanstd{34.2}{0.2} \\
KARE
& \meanstd{49.5}{0.3} & \meanstd{78.8}{0.3} & \meanstd{48.1}{0.2} & \meanstd{78.2}{0.3}
& \meanstd{80.7}{0.4} & \meanstd{32.9}{0.2} & \meanstd{81.3}{0.3} & \meanstd{33.5}{0.2} \\
ColaCare
& \meanstd{50.1}{0.4} & \meanstd{79.0}{0.3} & \meanstd{47.5}{0.4} & \meanstd{78.2}{0.3}
& \meanstd{82.3}{0.4} & \meanstd{35.5}{0.4} & \meanstd{82.5}{0.2} & \meanstd{33.7}{0.2} \\
\bottomrule
\end{tabularx}
}
\vspace{1pt}

\begin{tabularx}{\textwidth}{lYYYYYY}
\toprule
\multirow{3}{*}{\textbf{Model (\%)}} &
\multicolumn{6}{c}{\textbf{Task 3: Drug Recommendation}} \\
\cmidrule(lr){2-7}
& \multicolumn{3}{c}{\textbf{MIMIC-III}} &
  \multicolumn{3}{c}{\textbf{MIMIC-IV}} \\
\cmidrule(lr){2-4}\cmidrule(lr){5-7}
& \textbf{AUPRC} & \textbf{F1} & \textbf{Jaccard}
& \textbf{AUPRC} & \textbf{F1} & \textbf{Jaccard} \\
\midrule



\rowcolor{methodgreen}
PLM (BCMBERT-396M)
& \underline{\meanstd{79.6}{0.2}} & \underline{\meanstd{68.4}{0.3}} & \meanstd{51.5}{0.3}
& \underline{\meanstd{75.4}{0.1}} & \underline{\meanstd{64.8}{0.3}} & \meanstd{48.9}{0.3} \\

\rowcolor{methodgreen}
PLM (Meerkat-8B)
& \textbf{\meanstd{80.3}{0.2}} & \textbf{\meanstd{69.4}{0.3}} & \textbf{\meanstd{53.0}{0.3}}
& \textbf{\meanstd{77.1}{0.2}} & \textbf{\meanstd{65.7}{0.4}} & \textbf{\meanstd{50.4}{0.4}} \\
\midrule

Deepr
& \meanstd{72.3}{0.1} & \meanstd{60.3}{0.4} & \meanstd{44.7}{0.3}
& \meanstd{63.7}{0.1} & \meanstd{53.1}{0.4} & \meanstd{39.8}{0.4} \\

RETAIN
& \meanstd{75.1}{0.3} & \meanstd{65.2}{0.2} & \meanstd{49.4}{0.2}
& \meanstd{65.7}{0.6} & \meanstd{56.9}{0.4} & \meanstd{41.5}{0.4} \\

GRAM
& \meanstd{74.7}{0.1} & \meanstd{62.9}{0.3} & \meanstd{47.9}{0.3}
& \meanstd{65.3}{0.2} & \meanstd{53.1}{0.2} & \meanstd{40.3}{0.3} \\

StageNet
& \meanstd{73.4}{0.1} & \meanstd{61.4}{0.3} & \meanstd{45.8}{0.4}
& \meanstd{63.1}{0.1} & \meanstd{52.2}{0.3} & \meanstd{37.5}{0.4} \\

SafeDrug
& \meanstd{75.8}{0.2} & \meanstd{66.2}{0.2} & \meanstd{50.5}{0.2}
& \meanstd{67.1}{0.3} & \meanstd{58.2}{0.3} & \meanstd{43.0}{0.2} \\

MICRON
& \meanstd{76.5}{0.3} & \meanstd{67.4}{0.3} & \meanstd{51.1}{0.2}
& \meanstd{66.6}{0.4} & \meanstd{59.5}{0.3} & \meanstd{44.1}{0.3} \\

GAMENet
& \meanstd{76.1}{0.1} & \meanstd{66.0}{0.1} & \meanstd{50.2}{0.2}
& \meanstd{67.2}{0.3} & \meanstd{58.7}{0.3} & \meanstd{43.4}{0.3} \\

MoleRec
& \meanstd{72.8}{0.1} & \meanstd{61.1}{0.3} & \meanstd{46.1}{0.3}
& \meanstd{62.3}{0.1} & \meanstd{56.3}{0.4} & \meanstd{41.2}{0.3} \\

UDC
& \meanstd{78.5}{0.3} & \meanstd{67.5}{0.3} & \meanstd{51.7}{0.2}
& \meanstd{71.7}{0.2} & \meanstd{62.7}{0.3} & \meanstd{47.8}{0.3} \\

KEHGCN
& \meanstd{78.8}{0.2} & \meanstd{68.1}{0.3} & \underline{\meanstd{52.2}{0.2}}
& \meanstd{72.3}{0.4} & \meanstd{64.5}{0.3} & \underline{\meanstd{49.1}{0.3}} \\

G-BERT
& \meanstd{69.0}{0.2} & \meanstd{61.9}{0.2} & \meanstd{45.8}{0.3}
& \meanstd{62.4}{0.3} & \meanstd{51.7}{0.2} & \meanstd{37.3}{0.2} \\

LEADER
& \meanstd{78.0}{0.2} & \meanstd{67.4}{0.3} & \meanstd{51.8}{0.1}
& \meanstd{71.2}{0.2} & \meanstd{63.0}{0.2} & \meanstd{47.8}{0.2} \\

GLEM
& \meanstd{73.9}{0.1} & \meanstd{62.1}{0.4} & \meanstd{46.1}{0.4}
& \meanstd{66.0}{0.5} & \meanstd{54.5}{0.4} & \meanstd{41.4}{0.3} \\

GraphCare
& \meanstd{78.5}{0.2} & \meanstd{66.2}{0.3} & \meanstd{49.8}{0.4}
& \meanstd{70.7}{0.5} & \meanstd{60.4}{0.3} & \meanstd{45.7}{0.4} \\
\bottomrule
\end{tabularx}
\vspace{-1\baselineskip} 
\end{table*}

\subsection{Implementation Details}
\label{sec:imple}

To demonstrate that our method is backbone-agnostic and consistently improves performance across backbones, we instantiated the PLM framework with both encoder-only and decoder-only biomedical LMs.
For encoder-only backbone, we used \textit{BioClinical ModernBERT-396M (BCMBERT)} \citep{BCMBERT} and fine-tuned the last six Transformer layers with a task-specific prediction head. For decoder-only backbone, we used a Llama-based LLM \textit{Meerkat-8B}, obtained sequence embeddings via LLM2Vec~\citep{llm2vec}, and fine-tuned the model using LoRA adapters on attention projection layers~\citep{hu2022lora}, while freezing the remaining parameters. We used a standard three-layer GCN on the patient graph. Details about model architecture and hyperparameters are provided in Appendix~\ref{app:impl}.

\section{Results}
\label{sec:results}

\begin{figure*}[t]
\centering
\includegraphics[width=0.95\linewidth]{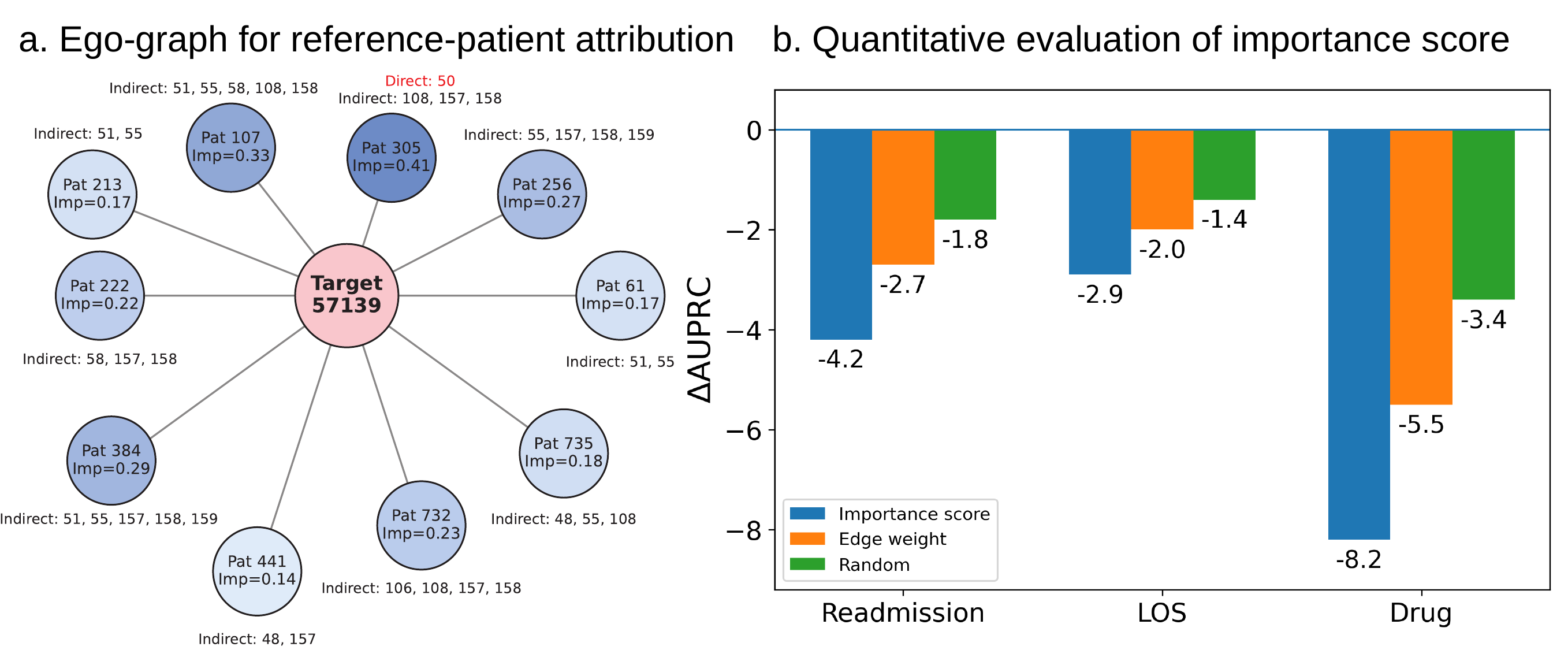}
\captionsetup{skip=3pt}
\caption{\textbf{Interpretability of PLM.}
\textbf{(a) Ego-graph for reference-patient attribution given a diabetes patient.} Nodes are positioned closer and colored darker for higher importance scores. \textcolor{red}{Direct} diabetes evidence is highlighted in red, indirectly related endocrine conditions and diabetes complications are also annotated.
\textbf{(b) Quantitative evaluation of importance scores.}
We assess explanation faithfulness using a masking-based evaluation, where more negative $\Delta$AUPRC indicates that the masked reference patients were more influential.}
\label{fig:interpret_ana}
\vspace{-1\baselineskip} 
\end{figure*}

\begin{figure*}[t]
  \centering
  \includegraphics[width=0.95\linewidth]{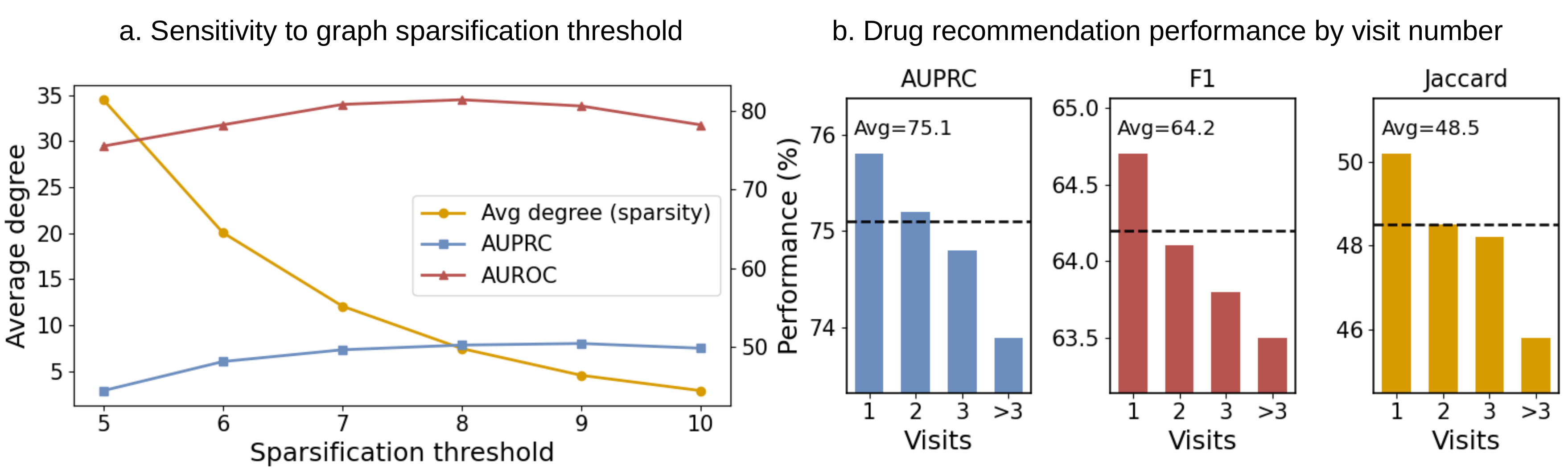}
    \caption{\textbf{Sensitivity analyses.}
    \textbf{(a) Sensitivity to the graph sparsification threshold $\tau$.}  Increasing $\tau$ produces a sparser graph with lower average degree. The readmission prediction performance peaks around $\tau=8$.
    \textbf{(b) Drug recommendation performance stratified by visit number on MIMIC-IV.} Performance remains stable for one to three visits and declines noticeably for more than three visits, indicating PLM does not rely on long visit histories.}
    \label{fig:sensitivity_analy}
    \vspace{-1\baselineskip} 
\end{figure*}
\begin{table*}[t]
\centering
\caption{Ablation studies on MIMIC-IV comparing PLM with four training strategies under the BCMBERT and Meerkat backbones. VEM consistently yields the best performance across tasks.}
\label{tab:ablation}
\scriptsize
\resizebox{\textwidth}{!}{%
\begin{tabular}{llccccccc}
\toprule
\textbf{Backbone} & \textbf{Training} &
\multicolumn{2}{c}{\textbf{Readmission}} &
\multicolumn{2}{c}{\textbf{LOS}} &
\multicolumn{3}{c}{\textbf{Drug Recommendation}} \\
\cmidrule(lr){3-4}\cmidrule(lr){5-6}\cmidrule(lr){7-9}
& & \textbf{AUPRC} & \textbf{AUROC}
& \textbf{AUPRC} & \textbf{F1}
& \textbf{AUPRC} & \textbf{F1} & \textbf{Jaccard} \\
\midrule
\rowcolor{methodgreen}
BCMBERT-396M & VEM
& \textbf{48.6}  & \textbf{80.6} & \textbf{82.8}  & \textbf{35.0}  & \textbf{75.4}  & \textbf{64.8}  & \textbf{48.9}  \\
        & Alternating & 45.6 & 76.9 & 79.8 & 33.4 & 69.2 & 58.2 &  44.8 \\
        & E2E    
& 44.2  & 76.0  & 78.4  & 31.6  & 71.5  & 62.5  & 46.4 \\
        & 2-stage 
& 45.5 & 77.4 & 78.7 & 32.0 & 72.4 & 62.8 & 46.3 \\
        & LM-only
& 43.6 & 74.5 & 77.0  & 30.7  & 70.2 & 61.4 & 45.5 \\
\midrule
\rowcolor{methodgreen}
Meerkat-8B & VEM 
& \textbf{49.7}  & \textbf{80.3}  & \textbf{85.3}  & \textbf{35.6}  & \textbf{77.1}  & \textbf{65.7}  & \textbf{50.4} \\
        & Alternating & 46.4 & 77.6 & 81.2 & 33.1 & 72.4 & 60.9 &  47.1 \\
        & E2E
& 44.4  & 77.2  & 78.0  & 30.8  & 70.4  & 60.2  & 45.4  \\
       & 2-stage 
& 46.2 & 77.3 & 81.8 & 32.9 & 73.2 & 63.6 & 47.9 \\
       & LM-only 
& 43.9 & 75.4 & 79.1  & 30.9  & 71.3  & 60.7  & 45.8 \\
\bottomrule
\end{tabular}%
}
\end{table*}

\subsection{Clinical Predictions on MIMIC Datasets}
\label{subsec:clinical}

As shown in Table~\ref{tab:results_mimic34}, PLM achieves the strongest overall performance across datasets and tasks. Both PLM variants consistently outperform prior clinical prediction models across readmission, LOS, and drug recommendation, with the Meerkat variant performing best in most settings. Meanwhile, PLM with BCMBERT remains highly competitive and still surpasses the strongest prior baselines on MIMIC-IV. The gains are particularly consistent for drug recommendation, suggesting that combining patient-level semantics with cohort structure benefits multi-label clinical prediction.

Beyond the transductive setting, we further evaluate PLM in an inductive setting, where each test patient independently connects to a fixed historical training cohort. As shown in Table~\ref{tab:inductive}, PLM achieves performance close to the transductive setting and outperforms the reported patient-independent baselines across all three tasks.

Table~\ref{tab:efficiency_cost} reports runtime and peak GPU memory on MIMIC-IV with the Meerkat-8B backbone, showing that VEM achieves substantially better performance than all baselines. VEM outperforms alternating optimization while requiring close training time and GPU memory. Compared with LM-only, VEM adds only 4.1\% runtime and 4.7\% memory overhead, while remaining more efficient than E2E and two-stage training. Further details are provided in Section~\ref{subsec:efficiency_analy}.

\subsection{Generalization Across LM Backbones}
\label{subsec:backbone_generalization}

To assess the generalization of VEM across backbones, we further evaluated PLM with four backbones on MIMIC-IV spanning both encoder-only models (BCMBERT and BioBERT) and decoder-only models (Meerkat and BioMistral). VEM consistently improves over each corresponding LM-only baseline across all three tasks, demonstrating that the proposed framework generalizes across different LM backbones rather than depending on a specific backbone. Detailed experimental results are provided in Appendix~\ref{app:backbone_generalization}.

\subsection{Interpretability Analysis}
\label{subsec:interpretability}

We assessed interpretability of PLM on insulin prediction in drug recommendation. Fig.~\ref{fig:interpret_ana}.a shows an ego graph for a target diabetes patient 57139 with ten retrieved reference cases. PLM assigns each reference patient an importance score (Imp); higher scores denote closer distance to the target and darker color. We highlighted in red the CCS codes directly associated with insulin (e.g., \texttt{50: Diabetes mellitus with complications}) and annotated indirectly related CCS codes such as endocrine conditions and diabetes complications. Patient 305 with a direct diabetes CCS code has the highest Imp 0.41. Patients 107 and 384 also have high scores (0.33 and 0.29) given their multiple diabetes-related comorbidities. Overall, the retrieved references support insulin prediction by highlighting both direct diabetes cases and cases with diabetes-related 
complications.

We further compared against a baseline that ranks reference patients using only weighted edges $A_{ij}$ (Fig.~\ref{fig:imp_wpg}). This edge-weight baseline can miss clinically salient references. Despite being the most influential reference with a direct diabetes code, patient 305 has an edge weight of 9, only slightly above the neighborhood threshold of 8. Therefore, edge weights mainly reflect shared-code overlap, whereas PLM leverages both textual semantics and graph structure to yield relevant cases as explanations. Additional details and interpretability analyses are provided in Appendix~\ref{app:interpret}.

To quantitatively evaluate the faithfulness of our reference-patient attributions, we conducted a masking-based evaluation on the readmission, LOS, and drug recommendation tasks. For each test patient, we removed the three retrieved reference edges selected by a given masking strategy, ran inference on the updated test graph, and measured the resulting change in AUPRC. We considered three selection strategies: top-3 by our importance scores, top-3 by edge weights, and random selection. Fig.~\ref{fig:interpret_ana}.b shows that masking the top-3 references ranked by our importance scores consistently yields the largest performance drop across all three tasks. This suggests that the computed importance scores more faithfully identify the reference patients that drive the model's predictions.

\subsection{Sensitivity Analysis}
\label{subsec:sensitivity}

We study readmission sensitivity to the graph sparsification threshold $\tau$, where an edge is retained only if $A_{ij} \geq \tau$. As shown in Figure~\ref{fig:sensitivity_analy}.a, increasing $\tau$ prunes weaker links, lowers average degree, and improves performance up to $\tau{=}8$. Beyond that, the graph becomes overly sparse and performance degrades. We therefore use $\tau{=}8$, which balances efficiency and predictive performance.

We also stratify MIMIC-IV drug recommendation performance by visit count. Figure~\ref{fig:sensitivity_analy}.b shows stable performance for one to three visits and a clearer decline beyond three, suggesting that PLM does not rely on long histories for strong performance; this is consistent with prior work on increasing patient complexity \citep{Meng2024-rt}.


\subsection{Ablation Studies}
\label{subsec:ablation}

We ablated PLM under both BCMBERT and Meerkat backbones by comparing it with LM-only and LM-GNN variants trained using alternating, E2E, or two-stage methods, with implementation details provided in Appendix~\ref{app:baselines}. As shown in Table~\ref{tab:ablation}, PLM with VEM achieves the best performance across all tasks under both backbones. These results indicate that the improvements stem from VEM's probabilistic training scheme, rather than simply adding a GNN, switching LM backbones,  or using alternative optimization methods.

\section{Conclusion}

We propose PLM, a unified LM-GNN framework for clinical prediction that integrates patient textual semantics with graph relational structure through a VEM algorithm. Across MIMIC-III and MIMIC-IV, PLM consistently improves readmission prediction, LOS prediction, and drug recommendation, and these gains generalize across both encoder-only and decoder-only LM backbones. Compared with other LM-GNN training strategies, VEM achieves these improvements with only modest computational overhead. Qualitative case studies show that PLM identifies influential similar patients as reference-patient explanations, while quantitative edge masking confirms that the highest-ranked references have the strongest effect on predictions.
\clearpage
\section*{Limitations}

PLM currently constructs patient graphs using diagnosis-code overlap, which may not capture richer clinical relationships available from multimodal EHR data or external knowledge. In addition, our reference-patient attributions are evaluated quantitatively for faithfulness, while further clinician assessment is needed before real-world deployment. Future work will explore richer graph construction and evaluate these attributions with clinical stakeholders.

\nocite{langley00}
\bibliography{example_paper}

\newpage
\appendix
\section{Detailed Mathematical Derivations}
\subsection{Derivation of Variational EM Algorithm}

We define the ELBO objective of the VI as 
\[
    \mathcal{L}(\theta) := \E_{q_\theta(y \mid x)} \Big[\log p_\phi (x, y) - \log q_\theta(y \mid x ) \Big].
\]
For the variational distribution $q_\theta(y \mid x)$, we adopt the mean-field assumption by factorizing latent variable $z$ independently:
\[
    q_\theta (y  \mid x) = \prod_{n=1}^N q_\theta(y_n \mid x_n) .
\]
Formally, maximizing the ELBO with respect to the variational distribution $q_\theta(y \mid x)$ is equivalent to minimizing the KL divergence between the posterior distribution and the variational distribution.
The KL divergence is defined as
\[
    \begin{aligned}
    &\mathrm{KL} \Big[ q_\theta(y \mid x) || p_\phi(y \mid x) \Big] \\
    &\quad = \E_{q_\theta(y \mid x)} \Big[\log q_\theta(y \mid x) -\log p_\phi(y \mid x) \Big] \\
    &\quad = \sum_y q_\theta(y \mid x) \log \frac{q_\theta(y  \mid x)}{p_\phi(y \mid x)}.
    \end{aligned}
\]
Therefore, in the E-step, we update the variational distribution $q_\theta(y \mid x)$ (i.e., fine-tuning the LM) by minimizing $\mathrm{KL} \Big[q_\theta(y \mid x) || p_\phi(y \mid x)\Big]$. However, this objective is typically intractable because it requires computing the expectations under $q_\theta(\cdot)$ over all possible values. To address this, we use the wake-sleep algorithm~\citep{Hinton1995-wi} to minimize the reverse KL
divergence, where the sleep phase minimizes the KL divergence the wrong way round:
\[
    \begin{aligned}
    \theta^* &:=\arg \min_\theta \mathrm{KL} \Big[p_\phi(y \mid x) \| q_\theta(y \mid x)\Big] \\
    &= \arg\min_\theta \E_{p_\phi(y\mid x)}\Big[\log p_\phi(y\mid x) \\
    &\qquad\qquad\qquad -\log q_\theta(y\mid x)\Big] \\
    &= \arg\max_\theta \E_{p_\phi(y\mid x)}\Big[\log q_\theta(y\mid x)\Big] + \const \\
    &= \arg \max_\theta \E_{p_\phi(x, y)} \Big[\log q_\theta(y \mid x)\Big] ,
    \end{aligned}
\]
where $\const=\E_{p_\phi(y\mid x)} \Big[\log p_\phi(y\mid x)\Big]$ does not depend on $\theta$ and is therefore omitted during optimization.
\begin{align} 
    \mathcal{L}_q &= \E_{p_\phi(x, y)} \Big[\log q_\theta(y \mid x)\Big] \nonumber \\
    & = \E_{p_\phi(x, y)} \Big[\sum_n^N \log q_\theta(y_n \mid x_n)\Big] \nonumber \\
    & = \E_{p_\phi(x, y)} \Big[\sum_n^N \log q_\theta(y_n \mid o_n)\Big] \nonumber\\
    & \approx \sum_n^N \log q_\theta\big(y_n \mid h_n^{(L)}(o, \bar{\phi})\big),
\end{align}
where $h_n^{(L)}(o;\bar{\phi})$ denotes the deterministic GNN representation obtained from the LM embeddings $o$ under frozen GNN parameters $\bar{\phi}$. 
Since the expectation over $p_\phi(x,y)$ is intractable, we estimate it with a single-sample Monte-Carlo sampling. In the E-step, LM first generates embeddings $o_n$, which are passed through the frozen GNN to obtain $h_n^{(L)}$. Because the GNN parameters are fixed, the loss is computed on $h_n^{(L)}(o,\bar{\phi})$ in place of $o_n$. Gradients then backpropagate through the fixed GNN to $o_n$, updating only the LM parameters $\theta$. Thus, optimizing $\mathcal{L}_q$ amounts to maximizing the label log-likelihood, which we approximate by replacing the LM embedding $o_n$ with the fixed GNN representation $h_n^{(L)}(o,\bar{\phi})$.

In the M-step, we then update the GNN model $p_\phi(y \mid x)$ towards maximizing the following pseudo-likelihood:
\begin{align}
    &\E_{q_\theta(y \mid x)} \Big[\log p(x, y) \Big] \nonumber \\
    &\quad = E_{q_\theta(y \mid x)} \Big[\log p_{\phi}(y \mid x) + \log p(x) \Big] \nonumber \\
    &\quad = \E_{q_\theta(y \mid x)}\Big[\log p_{\phi}(y \mid x) \Big] + \const  \nonumber \\
    &\quad \approx \E_{q_\theta(y \mid x)}\left[\sum_n^N  \log p_\phi(y_n \mid x, y_{\mathcal{N} \setminus n}) \right] \nonumber \\
    &\qquad\quad (\mathrm{pseudo likelihood}) \nonumber \\
    &\quad = \E_{q_\theta(y \mid x)}\left[\sum_n^N  \log p_\phi(y_n \mid x_n, y_{\mathcal{N} \setminus n}) \right]  \nonumber \\
    &\quad = \sum_{n=1}^N \log p_\phi (y_n \mid o_n,  y_{\mathcal{N} \setminus n}) \nonumber\\
    &\quad = \sum_{n=1}^N \log p_\phi(y_n \mid o_n, y_{\mathrm{NB}(n)}).
\end{align}

\subsection{Optimal Mean-Field Variational Update}
\label{app: optimum}

\textbf{Theorem~\ref{thm:q_opt}.} Given the fixed $p_\phi(y_n \mid x_n)$, the local optimum update $q_\theta^*(y_n \mid x_n)$ satisfies:
\begin{align}
    &\log q_\theta^*(y_n \mid x_n) \nonumber \\
    &\quad = \E_{q_\theta(y \mid x)} \Big[ \log p_\phi(y_n\mid x_n, y_{\mathrm{NB}(n)}) \Big].
\end{align}



\textbf{Proof} The goal of E-step is to optimize $q_\theta(y \mid x)$ by minimizing $\mathrm{KL} \big[q_\theta(y \mid x) || p_\phi(y \mid x)\big]$. Therefore, the objective function for $q_\theta(y \mid x)$ could be formulated as follows:
\begin{align}
    &\mathcal{L}(q_\theta(y_n \mid x_n)) \nonumber \\
    &\quad = -\mathrm{KL} \Big[q_\theta(y \mid x) \Vert p_\phi(y \mid x)\Big] \nonumber \\
    &\quad = \E_{q_\theta(y \mid x)} \Big[ \log p_\phi(y\mid x) - \log q_\theta(y \mid x) \Big] \nonumber\\
    &\quad = \E_{q_\theta(y \mid x)} \Big[ \log p_\phi(y \mid x) \nonumber\\
    &\qquad\qquad - \sum_{n'=1}^N \log  q_\theta(y_{n'} \mid x_{n'}) \Big] \nonumber\\
    &\quad = \E_{q_\theta(y \mid x)} \Big[\log p_\phi(y \mid x) \nonumber\\
    &\qquad\qquad -\sum_{n'\neq n}^N \log q_\theta(y_{n'} \mid x_{n'}) \nonumber\\
    &\qquad\qquad - \log q_\theta(y_n \mid x_n) \Big] \nonumber\\
    &\quad = \E_{q_\theta(y \mid x)} \big[ \log p_\phi(y \mid x) \nonumber\\
    &\qquad\qquad - \log q_\theta(y_n \mid x_n) \Big]  + \const \nonumber\\
    &\quad = \E_{q_\theta(y \mid x)} \Big[\log p_\phi(y_n \mid x, y_{\mathcal{N} \setminus n}) \nonumber\\
    &\qquad\qquad + \log p_\phi(y_{\mathcal{N} \setminus n} \mid x) \nonumber\\
    &\qquad\qquad - \log q_\theta(y_n \mid x_n) \Big] +\const \nonumber\\
    &\quad = \E_{q_\theta(y \mid x)} \Big[\log p_\phi(y_n \mid x, y_{\mathcal{N} \setminus n}) \nonumber \\
    &\qquad\qquad - \log q_\theta(y_n \mid x_n) \Big] + \const \nonumber \\
    &\quad = \E_{q_\theta(y \mid x)} \Big[\log p_\phi(y_n \mid x_n, y_{\mathrm{NB}(n)}) \nonumber \\
    &\qquad\qquad - \log q_\theta(y_n \mid x_n) \Big] + \const  \nonumber \\
    &\quad =-\mathrm{KL} \big[ q_\theta(y_n \mid x_n) \Vert p_\phi(y_n \mid x_n, y_{\mathrm{NB}(n)} )\big] \nonumber \\
    &\qquad\qquad + \const.
\end{align}
In GNNs, we leverage the Markov property that node $n$ is conditionally independent of all non-neighbors given its neighbors, reducing $y_{\mathcal{N} \setminus n}$ to $y_{\mathrm{NB}(n)}$.

As a result, the objective $O(q_\theta(y_n, \mid x_n))$ becomes the KL divergence between $q_\theta(y_n \mid x_n)$ and $p_\phi(y_n \mid x_n, y_{\mathrm{NB}(n)})$. he local optimum satisfies $\log q_\theta^*(y_n \mid x_n) = \E_{q_\theta(y_{-n} \mid x)} \Big[\log p_\phi(y_n \mid x_n, y_{\mathrm{NB}(n)})\Big] + \mathrm{const}$.  With Monte Carlo sampling, this local optimum becomes $\log q_\theta^*(y_n \mid x_n) \propto \log p_\phi(y_n \mid o_n, \tilde y_{\mathrm{NB}(n)})$, where $o_n$ is the LM‑encoded embedding for node $n$ and $\tilde y_{\mathrm{NB}(n)}$ denotes sampled neighboring labels. This indicates that optimal LM learning corresponds to using the GNN’s predictions for $y_n$ based on LM‑encoded embeddings $o_n$ and neighbor labels $y_{\mathrm{NB}(n)}$, integrating local textual information and global patient relationships into clinical prediction.

\subsection{Comparison with Classical EM and GLEM Under Full Supervision}
\label{app:glem_full_supervision}

The classical EM algorithm maximizes an observed-data likelihood by introducing an associated complete-data model.
In the classical EM formulation \citep{MLE_incomplete}, the observed data are fixed, whereas the complete data include components that are not directly observed. The E-step computes the conditional expectation of the complete-data log-likelihood given the observed data and the current parameter estimate, and the M-step maximizes this expected complete-data log-likelihood. Thus, classical EM is tied to an observed-data likelihood and a particular complete-data augmentation.

Our proposed VEM algorithm is related to this EM algorithm only at the level of its
alternating inference-prediction structure. It is not an incomplete-label
EM algorithm in the classical sense: all training labels are observed, and
we do not impute missing labels or recover unobserved components of the
training data. Instead, PLM introduces an auxiliary predictive distribution
to couple an LM-based semantic encoder with a GNN-based structured
predictor. This distinction is important because the variational object in
PLM is not a posterior distribution over missing ground-truth labels.

\begin{table*}[t]
\centering
\caption{Comparison between Patients-like-me and GLEM.}
\label{tab:plm_vs_glem}
\setlength{\tabcolsep}{5pt}
\renewcommand{\arraystretch}{1.2}
\begin{tabular}{p{0.20\linewidth}p{0.36\linewidth}p{0.36\linewidth}}
\toprule
Aspect & Patients-like-me & GLEM \\
\midrule
Objective 
& Supervised prediction of all labels $y$
& Semi-supervised completion of missing labels $y_U$ \\
\midrule
Observed variables
& Sequences $x$
& Node text features $s_V$ and labels on labeled nodes $y_L$ \\
\midrule
Latent variables
& Labels $y$ 
& Missing labels $y_U$ \\
\midrule
ELBO
& $\mathbb{E}_{q_\theta(y\mid x)}\big[\log p_\phi(x, y) - \log q_\theta(y\mid x)\big]$
& $\mathbb{E}_{q_\theta(y_U\mid s_U)}\big[\log p_\phi(y_L, y_U \mid s_V) - \log q_\theta(y_U\mid s_U)\big]$ \\
\midrule
E-step
& Generate patient embedding via LM $q_\theta(y \mid x)$ guided by graph structure 
& Fit LM $q_\theta(y_U \mid s_U)$ given GNN-imputed missing labels \\\midrule
M-step
& Predict targets via GNN $p_\phi(y \mid x)$ using LM embeddings
& Fit GNN $p_\phi$ using LM-imputed missing labels \\
\midrule
Under full supervision
& EM maintained via latent targets $y$ 
& EM collapse to joint training due to empty latent variables \\
\bottomrule
\end{tabular}
\end{table*}

\paragraph{Degeneracy of GLEM under full supervision.}
GLEM is formulated for semi-supervised node
classification on text-attributed graphs. Let \(V=L\cup U\) denote the set
of all nodes, where \(L\) is the labelled set and \(U\) is the unlabelled
set. The observed quantities are the node texts or semantic features
\(s_V\), the graph structure \(A\), and the labels \(y_L\) of the labelled
nodes. The labels \(y_U\) of the unlabelled nodes are treated as latent
variables. A generic GLEM-style semi-supervised variational objective can
be written as
\[
\begin{aligned}
&\mathcal{L}_{\mathrm{GLEM}}(\theta,\phi)  =
\mathbb{E}_{q_\theta(y_U\mid s_U)}
\Big[
\log p_\phi(y_L,y_U\mid s_V,A) \\
&\qquad\qquad\quad
-
\log q_\theta(y_U\mid s_U)
\Big],
\end{aligned}
\]
where \(q_\theta(y_U\mid s_U)\) is the LM-induced variational distribution
over the missing labels and \(p_\phi(y_L,y_U\mid s_V,A)\) is the
GNN-induced graph-structured predictive model.

\begin{proposition}[Collapse of GLEM's missing-label variational objective]
\label{prop:glem_degeneracy}
Suppose all nodes are labelled, so that \(U=\emptyset\) and \(L=V\). Then
the GLEM-style variational objective reduces to
\[
\mathcal{L}_{\mathrm{GLEM}}(\theta,\phi)
=
\log p_\phi(y_V\mid s_V,A).
\]
Consequently, the variational distribution over missing labels (i.e., latent variables) disappears, and the objective contains no nontrivial missing-label inference problem.
\end{proposition}

\begin{proof}
If all nodes are labelled, then \(U=\emptyset\) and hence \(y_U=\emptyset\).
The variational distribution over the empty label vector is the degenerate
distribution
\[
q_\theta(y_U\mid s_U)
=
q_\theta(\emptyset\mid s_\emptyset)
=
1.
\]
Therefore,
\[
-
\mathbb{E}_{q_\theta(y_U\mid s_U)}
\log q_\theta(y_U\mid s_U)
=
-\log 1
=
0.
\]
Moreover, the expectation with respect to \(q_\theta(y_U\mid s_U)\) is an
expectation over a point mass, so
\[
\begin{aligned}
&\mathbb{E}_{q_\theta(y_U\mid s_U)}
\left[
\log p_\phi(y_L,y_U\mid s_V,A)
\right] \\
&\qquad =
\log p_\phi(y_L\mid s_V,A).
\end{aligned}
\]
Since \(L=V\), this becomes
\[
\mathcal{L}_{\mathrm{GLEM}}(\theta,\phi)
=
\log p_\phi(y_V\mid s_V,A).
\]
Thus, the missing-label latent-variable component of the GLEM objective
vanishes under full supervision.
\end{proof}

Proposition~\ref{prop:glem_degeneracy} shows that the variational
interpretation of GLEM relies on the existence of unlabelled-node labels.
When \(U=\emptyset\), the latent space is
\[
\mathcal{Y}^{|U|}
=
\mathcal{Y}^{0}
=
\{\emptyset\},
\]
a singleton. Hence, there is no nontrivial E-step left for inferring
missing labels. Any remaining alternating update is then an implementation
choice for training LM and GNN components, rather than a non-degenerate
missing-label variational EM procedure.

\paragraph{Non-degeneracy of PLM under full supervision.}
PLM differs from GLEM in the role of the variational distribution. In supervised prediction, PLM introduces auxiliary predictive labels while ground-truth labels are available for supervision during training.
We denote the ground-truth  labels by $y=(y_1,\ldots,y_N)$, which enter the supervised losses. Separately, PLM introduces an
auxiliary predictive label vector $\tilde y=(\tilde y_1,\ldots,\tilde y_N)$, where $\tilde y$ denotes labels sampled from predictive distributions for LM-GNN coupling, distinct from the final predictions $\hat y$. The LM parameterizes a factorized auxiliary predictive distribution
\[
q_\theta(\tilde y\mid x)
=
\prod_{n=1}^N q_\theta(\tilde y_n\mid x_n),
\]
where $x=\{x_n\}_{n=1}^N$ denotes the patient sequences. The GNN parameterizes a graph-structured predictive distribution
\[
p_\phi(\tilde y\mid x,A).
\]

A supervised PLM objective may be written schematically as
\[
\begin{aligned}
&\mathcal{J}_{\mathrm{PLM}}(\theta,\phi) =
\mathcal{L}_{\mathrm{sup}}^{\mathrm{LM}}
(\theta;y)
+
\mathcal{L}_{\mathrm{sup}}^{\mathrm{GNN}}
(\phi;y) \\
&\qquad\quad
+
\gamma\,
\mathcal{L}_{\mathrm{coup}}(\theta,\phi),
\end{aligned}
\]
where \(\gamma\ge 0\) controls the strength of the LM--GNN coupling. The
coupling term is
\[
\begin{aligned}
&\mathcal{L}_{\mathrm{coup}}(\theta,\phi) =
\mathbb{E}_{q_\theta(\tilde y\mid x)}
\left[
\log p_\phi(\tilde y\mid x,A)
-
\log q_\theta(\tilde y\mid x)
\right].
\end{aligned}
\]

When exact maximization of the graph-structured likelihood is intractable,
PLM may replace the graph likelihood with a pseudo-likelihood approximation,
leading to
\[
\begin{aligned}
&\mathcal{L}_{\mathrm{coup}}^{\mathrm{PLM}}(\theta,\phi)  =
\mathbb{E}_{q_\theta(\tilde y\mid x)}
\Big[
\sum_{n=1}^{N}
\log p_\phi
\left(
\tilde y_n
\mid
x_n,\tilde y_{\mathcal{N}(n)},A
\right) \\
&\qquad\qquad\quad
-
\sum_{n=1}^{N}
\log q_\theta(\tilde y_n\mid x_n)
\Big],
\end{aligned}
\]
where \(\mathcal{N}(n)\) denotes the neighbours of node \(n\) in the
patient graph.

\begin{proposition}[Non-collapse of the PLM auxiliary variational space]
\label{prop:plm_nondegeneracy}
Assume that \(q_\theta(\tilde y\mid x)\) is an auxiliary predictive
distribution used to couple the LM and GNN components, rather than a
posterior distribution over unobserved ground-truth labels. Then full
supervision alone does not make the PLM variational space collapse. In
particular, the PLM coupling objective reduces to ordinary supervised
likelihood only if \(q_\theta(\tilde y\mid x)\) is additionally constrained
to be a point mass at the observed labels.
\end{proposition}

\begin{proof}
In GLEM, the variational distribution is defined only over the missing
labels \(y_U\) of unlabelled nodes. Thus, when \(U=\emptyset\), the
variational space collapses to a singleton. PLM instead defines auxiliary predictive distribution over
$\tilde y=(\tilde y_1,\ldots,\tilde y_N) \in \mathcal{Y}^{N}$.
This space remains nontrivial whenever \(N>0\) and \(|\mathcal{Y}|>1\).

Consequently, the entropy term
\[
-
\mathbb{E}_{q_\theta(\tilde y\mid x)}
\log q_\theta(\tilde y\mid x)
\]
is generally nonzero, and the coupling term
\[
\mathbb{E}_{q_\theta(\tilde y\mid x)}
\log p_\phi(\tilde y\mid x,A)
\]
continues to depend jointly on the LM-induced predictive distribution and
the GNN-induced structured predictive distribution. Hence, PLM's
objective contains a nontrivial LM-GNN alignment term even under full
supervision.

Only under the additional degenerate constraint
\[
q_\theta(\tilde y\mid x)
=
\delta_{y}(\tilde y)
\]
does the coupling expectation collapse to
\[
\mathcal{L}_{\mathrm{coup}}(\theta,\phi)
=
\log p_\phi(y \mid x,A),
\]
up to the zero entropy of the point mass. Therefore, PLM reduces to an
ordinary supervised graph-predictive likelihood only after imposing this
additional degeneracy condition on the auxiliary predictive distribution,
not merely because all training labels are observed.
\end{proof}

The distinction can therefore be summarized as follows. In GLEM, the
variational space is \(\mathcal{Y}^{|U|}\), which collapses to the singleton
\(\{\emptyset\}\) under full supervision. In PLM, the auxiliary predictive
space is \(\mathcal{Y}^{N}\), which remains nontrivial under full
supervision. Thus, full supervision removes the missing-label latent
variables in GLEM, but it does not remove the auxiliary predictive
distribution used by PLM to align the LM and GNN.

Accordingly, PLM should be viewed as a supervised VEM-inspired coupling
framework rather than as a classical incomplete-data EM algorithm. Its
novelty is not the use of EM itself, but the construction of a supervised
variational coupling objective that remains meaningful when all training
labels are observed.

\section{Details of Patients-Like-Me}

\subsection{Prompt Template}
\label{app:prompt}

\begin{tcolorbox}[
  enhanced,
  colback=white,
  colframe=gray!60,
  boxrule=0.9pt,
  arc=3mm,
  left=4mm,right=4mm,top=5mm,bottom=4mm,
  width=\linewidth,
  before upper={\raggedright},
  title=Prompt Template for Readmission Prediction,
  colbacktitle=gray!70,
  coltitle=white,
  fonttitle=\Large\rmfamily,
  boxed title style={boxrule=0pt,arc=3mm,left=4mm,right=4mm,top=2mm,bottom=2mm},
]
\noindent \textbf{Demographics:} Insurance: \ph{INSURANCE}; Language: \ph{LANGUAGE};\\
Religion: \ph{RELIGION}; Marital status: \ph{MARITAL\_STATUS}; Ethnicity: \ph{ETHNICITY}. \par 
\noindent The patient has \ph{VISIT\_NUM} ICU visits. \par 
\textcolor{blue}{\noindent In visit 1, diagnoses: \ph{DIAG\_NAME}, \ldots; procedures: \ph{PROC\_NAME}, \ldots; medications: \ph{MED\_NAME}, \ldots.}\par
\textcolor{blue}{\noindent \ldots}\par
\noindent Target visit: diagnoses: \ph{DIAG\_NAME}, \ldots; procedures: \ph{PROC\_NAME},  medications: \ph{MED\_NAME} \ldots. \\
\noindent \textbf{Task:} Predict whether the next recorded admission occurs within 15 days of the target visit's admission time.
Choose exactly one class from \{0,1\}, where 0 = the interval to the next admission exceeds 15 days and 1 = the interval is at most 15 days. \\
\noindent \textbf{Output format:} A single class label in \{0,1\}.
\end{tcolorbox}

\begin{tcolorbox}[
  enhanced,
  colback=white,
  colframe=gray!60,
  boxrule=0.9pt,
  arc=3mm,
  left=4mm,right=4mm,top=5mm,bottom=4mm,
  width=\linewidth,
  before upper={\raggedright},
  title=Prompt Template for Length-of-Stay Prediction,
  colbacktitle=gray!70,
  coltitle=white,
  fonttitle=\Large\rmfamily,
  boxed title style={boxrule=0pt,arc=3mm,left=4mm,right=4mm,top=2mm,bottom=2mm},
]
\noindent \textbf{Demographics:} Insurance: \ph{INSURANCE}; Language: \ph{LANGUAGE};\\
Religion: \ph{RELIGION}; Marital status: \ph{MARITAL\_STATUS}; Ethnicity: \ph{ETHNICITY}. \par
\noindent The patient has \ph{VISIT\_NUM} ICU visits. \par
\textcolor{blue}{\noindent In visit 1, diagnoses: \ph{DIAG\_NAME}, \ldots; procedures: \ph{PROC\_NAME}, \ldots; medications: \ph{MED\_NAME}, \ldots.}\par
\textcolor{blue}{\noindent \ldots}\par
\noindent Target visit: diagnoses: \ph{DIAG\_NAME}, \ldots; procedures: \ph{PROC\_NAME}, \ldots; medications: \ph{MED\_NAME}, \ldots. \par
\noindent \textbf{Task:} Predict the ICU length-of-stay (LOS) category for the target visit. \\
\noindent Choose exactly one class from \{0,1,2,3,4,5,6,7,8,9\}, where: \\
\noindent 0: LOS $<$ 1 day; \\
\noindent 1: LOS = 1 day; \\
\noindent 2: LOS = 2 days; \\
\noindent 3: LOS = 3 days; \\
\noindent 4: LOS = 4 days; \\
\noindent 5: LOS = 5 days; \\
\noindent 6: LOS = 6 days; \\
\noindent 7: LOS = 7 days; \\
\noindent 8: LOS = 8-14 days; \\
\noindent 9: LOS $\ge$ 15 days. \par 
\noindent \textbf{Output format:} A single class label in \{0,1,2,3,4,5,6,7,8,9\}.
\end{tcolorbox}

\begin{tcolorbox}[
  enhanced,
  colback=white,
  colframe=gray!60,
  boxrule=0.9pt,
  arc=3mm,
  left=4mm,right=4mm,top=5mm,bottom=4mm,
  width=\linewidth,
  before upper={\raggedright},
  title=Prompt Template for Drug Recommendation,
  colbacktitle=gray!70,
  coltitle=white,
  fonttitle=\Large\rmfamily,
  boxed title style={boxrule=0pt,arc=3mm,left=4mm,right=4mm,top=2mm,bottom=2mm},
]
\noindent \textbf{Demographics:} Insurance: \ph{INSURANCE}; Language: \ph{LANGUAGE};\\
Religion: \ph{RELIGION}; Marital status: \ph{MARITAL\_STATUS}; Ethnicity: \ph{ETHNICITY}. \par 
\noindent The patient has \ph{VISIT\_NUM} ICU visits. \par 
\textcolor{blue}{\noindent In visit 1, diagnoses: \ph{DIAG\_NAME}, \ldots; procedures: \ph{PROC\_NAME}, \ldots; medications: \ph{MED\_NAME}, \ldots.}\par
\textcolor{blue}{\noindent \ldots}\par
\noindent Target visit: diagnoses: \ph{DIAG\_NAME}, \ldots; procedures: \ph{PROC\_NAME}, \ldots. \\
\noindent \textbf{Task:} Predict the medications to prescribe for the target visit. Output a multi-label set of \textbf{3-level ATC codes}. Choose any number of codes from the medication vocabulary. Note that a patient may have multiple medications.\\
\noindent \textbf{Output format:} A multi-label set of ATC level-3 codes.
\end{tcolorbox}

\section{Additional Experimental Details}
\label{app:exp}
\subsection{Datasets and Cohort Statistics}
\label{app:dataset_stats}

We used two large-scale, de-identified EHR datasets MIMIC-III and MIMIC-IV from Beth Israel Deaconess Medical Center (BIDMC), released via PhysioNet under a credentialed data use agreement
\citep{johnson2016mimic,johnson2023mimiciv}. Both datasets contain inpatient EHR records, including diagnoses, procedures, medication prescriptions,
and admission/discharge timestamps. Patient identifiers are removed and all timestamps are shifted by a patient-specific offset. This preserves relative temporal ordering within each patient sequence while preventing alignment of absolute calendar time across patients. We construct visit-based patient sequences by treating each inpatient admission as a visit. Table~\ref{tab:mimic_basic_stats} summarizes the basic cohort statistics after preprocessing.

\begin{table*}[th]
\centering
\caption{Basic statistics of the preprocessed MIMIC-III and MIMIC-IV datasets. ``\#'' denotes counts and ``/patient'' denotes the average number of records per patient.}
\label{tab:mimic_basic_stats}
\small
\setlength{\tabcolsep}{5pt}
\resizebox{\textwidth}{!}{%
\begin{tabular}{lcccccc}
\toprule
Dataset
& \#patients
& \#visits
& \#visits/patient
& \#conditions/patient
& \#procedures/patient
& \#drugs/patient \\
\midrule
MIMIC-III
& 35,707 & 44,399 & 1.24 & 12.89 & 4.54 & 33.71 \\
MIMIC-IV
& 123,488 & 232,263 & 1.88 & 21.74 & 4.70 & 43.89 \\
\bottomrule
\end{tabular}%
}
\end{table*}
\subsection{Preprocessing and Concept Normalization}
\label{app:preprocess}

We applied the same preprocessing pipeline to both MIMIC-III and MIMIC-IV, including cohort selection, code normalization, and data splitting, to obtain consistent visit-based patient EHR sequences~\citep{graphcare2024}. We used the same normalization procedure across datasets to ensure a fair comparison.

\paragraph{Cohort construction and splitting.}
We treat each inpatient admission as a visit and represent each patient record as an ordered sequence of visits.
We perform patient-level train/validation/test splits so that all visits from the same patient appear in only one split, avoiding information leakage across splits. We randomly divide the dataset into training, validation and test sets in a 60\%: 20\%: 20\% ratio. After splitting patients, we independently construct $G_{\mathrm{train}}$, $G_{\mathrm{val}}$, and $G_{\mathrm{test}}$ using patients within each split.

\paragraph{Diagnosis and procedure codes.}
For each visit, we extract diagnoses and procedures from the corresponding MIMIC tables and represent them as unordered sets of codes.
We retain the original ICD coding systems (ICD-9-CM and ICD-10-CM for diagnoses; ICD-9 and ICD-10-PCS for
procedures) and preserve the recorded ICD version to avoid conflating codes across systems. We map ICD-9/10 diagnosis and procedure codes to CCS categories, and de-duplicate repeated codes within each visit. Specifically, 11,736 ICD-9-CM and 72,446 ICD-10-CM diagnosis codes are mapped to 285 CCS condition categories, while 4,670 ICD-9-PROC and 79,758 ICD-10-PCS procedure codes are mapped to 231 CCS procedure categories, yielding 285 condition and 231 procedure concepts in total.

\paragraph{Medication mapping and normalization.}
Medications are extracted from the prescriptions tables, where drugs are originally recorded as National Drug
Codes (NDC).
Following GraphCare~\citep{graphcare2024}, we standardize medications using a two-stage mapping pipeline:
\textbf{NDC $\rightarrow$ RxNorm (RxCUI) $\rightarrow$ ATC}.
We keep \textbf{ATC level-3} codes to preserve fine-grained medication semantics.
Unmappable NDC entries (e.g., non-drug supplies, institution-specific packaging codes, or missing references)
are excluded from the medication representation.
This yields 269 unique drugs represented as ATC level-3 codes.

\paragraph{Visit-level feature construction.}
Each visit $x_i$ is represented by three unordered sets: diagnosis codes, procedure codes, and mapped ATC-3
medication codes.
These per-visit sets are then composed into a visit-based patient history $\{x_1,\ldots,x_t\}$ by admission time.
Basic cohort statistics after preprocessing are reported in Appendix~\ref{app:dataset_stats}.

\paragraph{Task instance generation.}
Given a patient sequence $\{x_1,\ldots,x_t\}$, we generate training instances from visit prefixes $(x_{1:i})$, with prediction targets defined for each task in Appendix~\ref{app:tasks}. For each task instance, to avoid future information leakage, graph edges are computed using only diagnosis codes available within the corresponding input prefix.

\subsection{Task Definitions and Metric Computation}
\label{app:tasks}

We follow the task formulation and evaluation protocol of prior work~\citep{graphcare2024}. Given a patient visit sequence $\{x_1,\ldots,x_t\}$, we construct prediction inputs from visit prefixes $x_{1:i}=(x_1,\ldots,x_i)$, where $x_i$ is the target visit and eligible target visits are specified for each task. Unless otherwise stated, we report the average metrics and standard deviation over 10 independent runs, as shown in Table~\ref{tab:results_mimic34}.

\paragraph{Readmission prediction.} 
For each visit $x_i$ with $i \le t-1$, let $\tau(x_i)$ denote the admission time of visit $x_i$ (in the de-identified timeline, preserving within-patient time intervals). For a multi-visit patient sequence $(x_{1:i})$ with $i\le t-1$, we define the readmission label as
\begin{equation}
y^{\textsc{readm}}_{i}=\mathbb{I}\left[\tau(x_{i+1})-\tau(x_i)\le \sigma\right],
\end{equation}
where we set $\sigma=15$ days in our experiments. The interval is measured between consecutive admission times (admission-to-admission), rather than from discharge. Only visits with a subsequent recorded admission are included, excluding the final visit of each patient and patients with only one recorded visit. We evaluate readmission using AUPRC and AUROC on both MIMIC-III and MIMIC-IV datasets.

\paragraph{LOS prediction.}
LOS prediction is formulated as a $C$-category classification problem with $C=10$ classes, corresponding to stays of $<1$ day (0), $1$-$7$ days (1-7), $8$-$14$ days (8), and $\ge 15$ days (9).
For each visit $x_i$, the goal is to predict the LOS category of the target visit $x_{i}$. 
The input includes all preceding visits and the diagnosis, procedure, and medication codes recorded for the target visit, without a within-visit prediction cutoff. Thus, we evaluate retrospective visit-level LOS classification rather than prospective LOS prediction at admission.
We report AUPRC and (macro-)F1 scores on both MIMIC-III and MIMIC-IV datasets.

\paragraph{Drug recommendation.}
Drug recommendation is formulated as a multi-label classification task. Let $\mathcal{V}_{\textsc{drug}}$ denote the medication vocabulary and $S_i \subseteq \mathcal{V}_{\textsc{drug}}$ denote the medication set prescribed at the target visit $x_i$. The target is represented as a multi-hot vector
$y^{\textsc{drug}}_{i} \in \{0,1\}^{|\mathcal{V}_{\textsc{drug}}|}$. For  the target visit $x_i$, the model input consists of all EHR codes from prior visits $\{x_1,\ldots,x_{i-1}\}$, including their diagnosis, procedure, and medication codes, together with the diagnosis and procedure codes observed at the target visit $x_i$. The medication set prescribed at the target visit, denoted by $S_i$, is excluded from the input and serve as the prediction target.  We report AUPRC, F1, and Jaccard on both datasets. For F1 and Jaccard, we convert predicted probabilities to a binary set using a fixed threshold of $0.2$. Jaccard is computed per instance as $\frac{|S_i\cap \hat{S}_i|}{|S_i\cup \hat{S}_i|}$ and averaged over the test set.

\paragraph{Training objectives.}
We use binary cross-entropy (BCE) with sigmoid for readmission and drug recommendation, and cross-entropy (CE) with
softmax for the multi-class LOS task.

\subsection{Baselines and Reproducibility Notes}
\label{app:baselines}

We describe the baseline implementations and the evaluation protocol used to ensure fair comparison.

\paragraph{Baselines.}
We compare against representative patient-level sequence models, namely Deepr~\citep{nguyen2017deepr},
RETAIN~\citep{RETAIN}, GRAM~\citep{GRAM}, StageNet~\citep{gao2020stagenet}, AdaCare~\citep{Ma2020-gf}, and GRASP~\citep{zhang2021grasp}. We further compare with LM-graph models that incorporate external medical knowledge or graph, including G-BERT~\citep{shang2019gbert}, LEADER~\citep{liu2024leader}, GLEM~\citep{zhao2022glem}, GraphCare~\citep{graphcare2024}, KARE~\citep{KARE}, ColaCare~\citep{ColaCare}. For drug recommendation, we also include task-specific models SafeDrug~\citep{yang2021safedrug}, MICRON~\citep{yang2021micron},
GAMENet~\citep{shang2019gamenet}, MoleRec~\citep{yang2023molerec}, and UDC~\citep{UDC}.

\paragraph{Optimization Strategy Baselines.}

For the ablation and efficiency analyses, we compare VEM with four training strategies: LM-only, two-stage, end-to-end (E2E), and Alternating.
In two-stage training, the LM is first fine-tuned and then frozen, with the GNN subsequently trained on its fixed embeddings. In E2E training, the LM and GNN are jointly optimized through end-to-end backpropagation. In the Alternating baseline, the LM is first updated independently using ground-truth labels, followed by GNN updates on frozen LM embeddings.  In contrast, PLM couples the LM and GNN during optimization, updating the LM using predictions from the fixed GNN in the E-step and the GNN using the LM predictive distribution in the M-step. Thus, the key difference is that PLM couples the LM and GNN during optimization, while the Alternating baseline updates them separately.
VEM and other baselines  use the same initialization, update schedule, training budget, and early-stopping criterion.

\paragraph{Implementation sources and training protocol.}
Where official implementations are available, we follow the authors' released code and default settings.
Otherwise, we re-implement baselines based on the original papers.
All methods are trained and evaluated under the same preprocessing pipeline, patient-level data split,
and task instance construction (Appendix~\ref{app:preprocess}-\ref{app:tasks}).
Hyperparameters for each baseline are tuned on the validation set and the selected configuration is used for
final testing.
Unless stated otherwise, we use early stopping based on the validation metric corresponding to each task and
report results over multiple runs (Table~\ref{tab:results_mimic34}).

\paragraph{Input alignment across methods.}
For sequence-based baselines such as Deepr, RETAIN, GRAM, and StageNet, each visit is represented with the same
set-valued clinical concepts used by our method, and visits are ordered chronologically. LM/graph-enhanced baselines
such as G-BERT, LEADER, GLEM, and GraphCare are evaluated using the same normalized concept vocabularies and the
same patient-level split to ensure comparability.

\paragraph{Applicability constraints.}
AdaCare and GRASP are designed for binary clinical classification and do not directly apply to the multi-class LOS
or multi-label drug recommendation settings under our formulation; therefore, we report their results only on the
readmission task, consistent with Table~\ref{tab:results_mimic34}.

\paragraph{Drug-specific knowledge.}
For drug recommendation baselines that rely on drug-drug interaction (DDI) knowledge (e.g., SafeDrug and GAMENet),
we follow the original papers and corresponding implementations to construct the required drug graphs/knowledge,
while keeping the same ATC-level medication normalization used in our preprocessing (Appendix~\ref{app:preprocess}).
All methods are evaluated on the same target medication space.

\paragraph{Reproducibility.}
We fix random seeds for each run and report mean and standard deviation across 10 runs. 
Complete hyperparameters, training budgets, and hardware details are provided in Appendix~\ref{app:impl}.

\subsection{Implementation Details and Hyperparameters}
\label{app:impl}

This section summarizes model configurations, optimization settings, and reproducibility details.

\paragraph{Backbone language models.}
We instantiate our framework with two biomedical language models:
(i) an encoder-only \textit{BioBERT-large} model (\texttt{biobert-large-cased-v1.1}) and
(ii) a decoder-only \textit{BioMistral} model.
For BioBERT-large, we fine-tune the last six Transformer layers together with a task-specific prediction head,
while freezing all earlier layers.
For BioMistral, we extract sequence-level representations using LLM2Vec~\citep{llm2vec} and fine-tune the model
via parameter-efficient LoRA adapters, freezing all non-adapted parameters.

\paragraph{LLM2Vec configuration.}
We follow the standard LLM2Vec setup to obtain fixed-length sequence embeddings from the decoder-only backbone.
Specifically, we use the last-layer hidden states and apply mean pooling over the token dimension,
followed by a linear projection to match the hidden dimension of the downstream graph model.
This configuration is shared across all tasks and datasets.

\paragraph{LoRA configuration (BioMistral).}
LoRA adapters are applied to the attention projection layers of BioMistral.
Unless otherwise stated, we use a uniform configuration across tasks:
rank $r=8$, scaling factor $\alpha=16$, and LoRA dropout $0.1$.
Adapters are inserted into the $\{q\_proj, k\_proj, v\_proj, o\_proj\}$ modules of each attention block.

\paragraph{GNN.}
We use a standard three-layer GCN with a hidden dimension of 128 and ReLU activation, followed by a softmax classifier~\citep{kipf2017semi}. The GCN layer updates node representations based on a weighted patient graph with adjacency $A$, where $A_{ij}\ge 0$ encodes patient similarity:
\begin{equation}
\begin{split}
H^{(\ell+1)}&=\sigma\!\left(\tilde{D}^{-\frac12}\tilde{A}\tilde{D}^{-\frac12} H^{(\ell)} W^{(\ell)}\right), \\
\tilde{A}&=A+I,
\end{split}
\label{eq:gcn}
\end{equation}
where $\tilde{D}$ is the degree matrix of $\tilde{A}$.
Equivalently, for node $i$ (ignoring $\sigma$),
\begin{equation}
h_i^{(\ell+1)}=\sum_{j\in \mathcal{N}(i)\cup\{i\}}
\frac{\tilde{A}_{ij}}{\sqrt{\tilde{d}_i\,\tilde{d}_j}}\; h_j^{(\ell)} W^{(\ell)} .
\label{eq:gcn_nodewise}
\end{equation}
This formulation naturally supports weighted patient graphs: the edge weight $A_{ij}$ enters directly through $\tilde{A}_{ij}$ and modulates the strength of message passing between patients.

Since the LM produces higher-dimensional embeddings (e.g., 768 for BioBERT and 4096 for BioMistral), we use an MLP projection to map them to 128 as the GCN input.

\paragraph{Optimization and training.}
All models are optimized using AdamW.
We use separate learning rates for the language model and the graph/prediction components:
$\eta_{\textsc{lm}} = 1\times10^{-5}$ and $\eta_{\textsc{gnn}} = 1\times10^{-3}$.
Weight decay is set to $1\times10^{-2}$.
We train with batch size $32$ and use gradient accumulation of $2$ steps when necessary.
The maximum input sequence length is capped at $512$ tokens.
We adopt a linear learning-rate scheduler with a warmup ratio of $10\%$ of the total training steps.
Early stopping is applied based on the primary validation metric of each task with a patience of $5$ epochs.

\begin{algorithm}[t]
\caption{VEM Optimization}
\label{alg:optim}
\footnotesize
\KwIn{Patient sequences $x$, target labels $y$, LM $q_\theta$, GNN $p_\phi$, and patient graph $A$}
\KwOut{LM embeddings $\{o_n\}_{n=1}^N$, GNN node embeddings $\{h_n^{(L)}\}_{n=1}^N$, and predictions $\{\hat{y}_n\}_{n=1}^N$}
\While{not converged}{
  \tcp{E-step: update the LM while fixing the GNN}
  update LM $q_\theta$ by maximizing $\mathcal{L}_q$ in \eqref{eq:e-step-obj};
  \tcp{M-step: update the GNN while fixing the LM}
  update GNN $p_\phi$ by maximizing $\mathcal{L}_p$ in \eqref{eq:m-step-obj};
}
\end{algorithm}

\paragraph{Training budget and hardware.}
All experiments are conducted on NVIDIA Dgx Spark GPU with 128GB of memory.
We train models for up to $10$ epochs with early stopping.
Mixed-precision training (FP16) is used to improve training efficiency.

\paragraph{Random seeds and reporting.}
We use 10 random seeds for 10 runs, with one seed assigned to each run. The same seeds are used across all methods and baselines, and we report the mean and standard deviation over the 10 runs.

\subsection{Details of the Weighted Patient Graph}
\label{app:weighted_graph}

We provide additional details on the construction of the weighted patient graph. Following prior work~\citep{Lu2021-re}, we first represent the cohort as a bipartite graph between patients and diagnosis codes. A bipartite graph is a special class of graph consisting of two disjoint sets of vertices. In this study, we use an undirected bipartite graph to represent relationships between patients and diseases. Here, we project the bipartite graph into the patient side, to construct a  weighted patient network: two patients are connected if they share at least one diagnosed disease, and the edge weight equals the number of diseases they have in common (i.e., their number of shared neighbors). To avoid counting repeated occurrences, we compute edge weights using unique diagnosis codes per patient.
As illustrated in Fig.~\ref{fig:outline}b, patients $1$ and $2$ have one EHR code, resulting in a weight of 1 between them; patients $3$ and $4$ have a weight of 2 due to two shared EHR codes. In the projected graph, each patient node retains its original properties, and two patients are connected if they share at least one neighbor (i.e., a diagnosis code) in the original bipartite graph. Prior studies suggest that diseases that co-occur can reflect shared biological mechanisms (e.g., via disease–gene interactions) and that patients with the same chronic conditions often share common risk factors (e.g., smoking history, obesity, and insufficient physical activity)~\citep{Lu2021-re}. Motivated by these observations, we use the weighted patient network to capture latent patient–patient relationships induced by shared clinical profiles.

We deliberately construct the patient graph from diagnosis overlap only. This design reduces potential target leakage through graph construction: incorporating current-visit medications would directly leak targets for drug recommendation, while incorporating procedures or other current-visit signals could partially reveal the LOS target. We therefore use diagnosis overlap as a simple, task-agnostic cohort prior that remains compatible with all three prediction settings. Richer graph signals, such as temporality, disease severity, or demographics, may further improve graph quality, but they are not straightforward to use with the MIMIC datasets.

\begin{figure}[t]
  \centering
  \includegraphics[width=\linewidth]{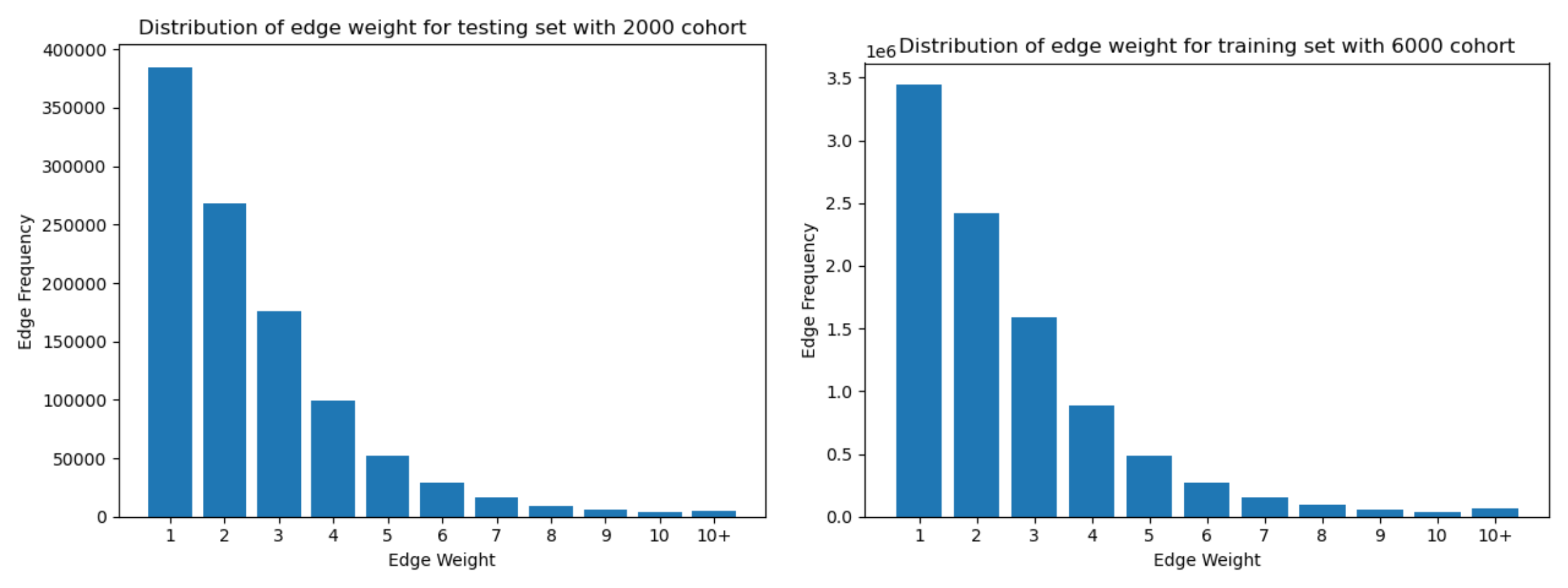}
  \caption{Edge-weight distributions in the weighted patient graphs for the training cohort (6{,}000 patients) and test cohort (2{,}000 patients). Edge weights are defined as the number of shared diagnosis codes between two patients and are binned as 1-10 and 10+. The two cohorts exhibit highly similar distributions with Jensen–Shannon Divergence of 0.0097 and Total Variation Distance of 0.0052, suggesting minimal practical distribution shift.}
  \label{fig:edge_weight_dist}
\end{figure}

Fig.~\ref{fig:edge_weight_dist} shows the edge-weight distributions of the weighted patient graphs constructed from 6,000 training patients and 2,000 test patients sampled from MIMIC-III. The histogram is heavily skewed toward small weights (primarily 1–3 shared diagnosis codes), and the frequency drops rapidly as the weight increases. This long-tailed pattern indicates a sparse graph in which most patient pairs are only weakly connected, with a small minority exhibiting strong similarity through many shared diagnoses. We compare the edge-weight distributions of the training and test patient graphs using a G-test on binned weights. The distributions are highly similar, with Jensen–Shannon Divergence of 0.0097 and Total Variation Distance of 0.0052, indicating minimal distribution shift. This indicates similar graph sparsity and edge-weight patterns across the training and test cohorts.

To keep the resulting patient graph sparse, we further sparsify the projected graph by retaining an edge only when its weight exceeds a minimum threshold, i.e., $A_{ij}\ge \tau$. Based on the distribution in Fig.~\ref{fig:edge_weight_dist}, we set $\tau=8$ to filter out the large number of weak connections. This choice yields a substantially sparser patient graph, with an average node degree on the order of $\sim$8 in the sampled cohort, while preserving clinically meaningful links among genuinely similar patients.

The training, validation, and test graphs are constructed independently.
The model is trained only on $G_{\mathrm{train}}$. During testing, all
test patients are encoded by the fixed LM and jointly processed by the
fixed GNN on $G_{\mathrm{test}}$. Test--test edges therefore enable
representation-level message passing, while test labels are never used
during inference.

\subsection{Additional Details on Interpretability Analysis}
\label{app:interpret}

\begin{figure*}[t]
  \centering
  \includegraphics[width=\linewidth]{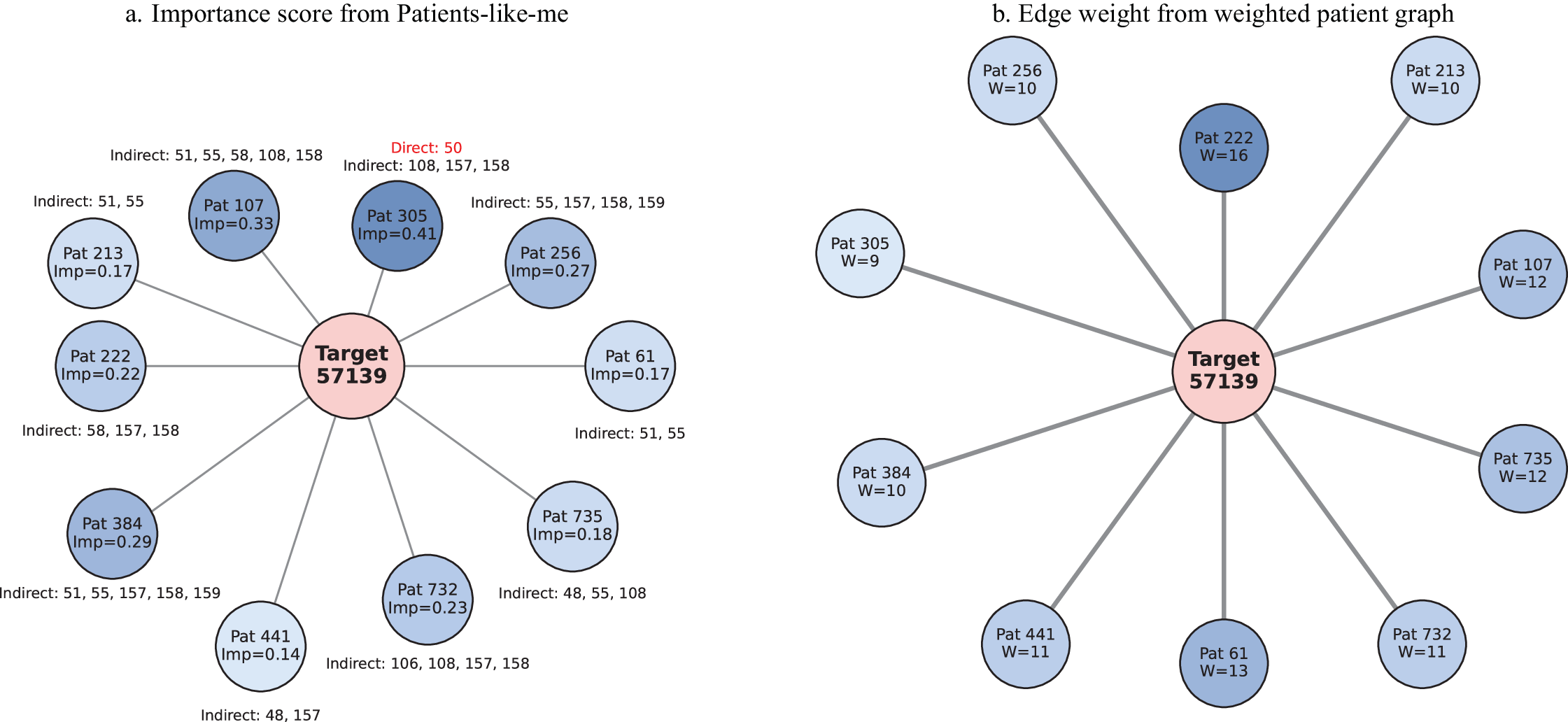}
  \caption{Comparison of two explanation signals for the same target patient (\texttt{57139}). 
  \textbf{(a)} \emph{Patients-like-me} ranks reference patients and assigns a patient-level importance score (Imp) indicating each reference patient's contribution to the target prediction.
  \textbf{(b)} Baseline visualization using the weighted patient graph, where the edge weight $A_{ij}$ equals the number of shared diagnosis codes between the target patient and each reference patient.}
  \label{fig:imp_wpg}
\end{figure*}

Figure~\ref{fig:imp_wpg} contrasts two explanations for the same target patient 57139 in the insulin drug recommendation task. We visualize an ego-network comprising the target node and its 10 reference patients retrieved by Patients-like-me. Each reference node is labeled with its importance score (Imp). The higher Imp is encoded by a shorter edge length (closer to the target) and a darker blue node color, enabling qualitative comparison across references.

In panel (a), the target patient is placed at the center and the 10 reference patients are arranged around it for clarity. Each reference node is annotated with an importance score that quantifies its influence on the target prediction. To improve interpretability, we additionally annotate  the patient-level clinical evidence using CCS codes:
(i) \textcolor{red}{direct diabetes CCS evidence} that is strongly associated with insulin use (e.g., \texttt{50: Diabetes mellitus with complications}), and 
(ii) indirectly related CCS categories capturing endocrine conditions or diabetes-related complications (e.g., chronic kidney disease). 
These annotations are extracted from the mapped CCS codes in each patient’s structured record and shown near each reference node. Panel (b) provides a structural baseline explanation derived from the weighted patient graph used by our model. We use the same target patient and the same set of 10 reference patients, but replace the Patients-like-me importance score with the edge weight $A_{ij}$ between patients $i$ and $j$. We visualize the same ego graph, encoding the edge weight $A_{ij}$ by edge length (larger $A_{ij}$ yields a shorter edge, indicating stronger structural proximity), and using node color to reflect the same weight for consistency. This structural baseline reflects a common assumption in cohort graphs: patients are considered more related if they share more structured codes.

In panel (a), Patients-like-me assigns the highest importance to patient 305 with an importance score of 0.41, which is clinically aligned with the presence of direct diabetes CCS code 50. Moreover, patients 107 and 384 also receive high importance scores (0.33 and 0.29), consistent with multiple indirectly related endocrine conditions and diabetes complications annotated near those nodes. Other reference patients have fewer indirectly related CCS codes and therefore receive lower importance scores. In contrast, the structural baseline in panel (b) ranks references purely by code overlap, which can yield less informative explanations. For example, patient 305, although the most influential reference in panel (a), has only an edge weight of 9, only slightly above the sparsification threshold. This indicates that edge weights alone only capture structural shared-code overlap and can under-emphasize clinically salient evidence that drives the model’s prediction. Overall, Fig.~\ref{fig:imp_wpg} shows that Patients-like-me provides explanations by retrieving reference patients that are not only structurally similar in the cohort graph but also clinically and semantically relevant to the target prediction. This comparison highlights that our patient-level importance scores are more informative than explanations based only on graph connectivity.

\subsection{Monte Carlo Sampling Analysis}
\label{app:sampling_analy}

We examine the sensitivity to the number of Monte Carlo samples used to estimate the E-step expectation in Eq.~\ref{eq:e-step-obj}. Sampling is performed over auxiliary predictive labels drawn from the fixed GNN distribution $p_\phi$, while the LM embedding $o_n$ and GNN representation $h_n^{(L)}$ are deterministic forward-pass outputs. We compare single-sample and five-sample estimators on the MIMIC-IV readmission task with the BioBERT backbone under the same training configuration.

As shown in Table~\ref{tab:mc_samples}, increasing the number of samples from one to five improves AUPRC by 1.2 percentage points and AUROC by 0.7 percentage points. However, training time increases from 0.24 to 1.05 hours per epoch (4.4$\times$). We therefore adopt a single sample by default to balance predictive performance and computational efficiency.

\begin{table}[H]
\centering
\caption{\textbf{Sensitivity to the number of Monte Carlo samples in the E-step.} Results are reported on the MIMIC-IV readmission task with the BioBERT backbone. Training time is measured in hours per epoch.}
\label{tab:mc_samples}
\small
\setlength{\tabcolsep}{3pt}
\renewcommand{\arraystretch}{1.1}
\begin{tabular}{lccc}
\toprule
\textbf{Number} & \textbf{AUPRC (\%)} & \textbf{AUROC (\%)} & \textbf{Time/Epoch (h)} \\
\midrule
Single (Ours) & 47.8 ($\pm$ 0.4) & 80.1 ($\pm$ 0.4) & 0.24 \\
Five & 49.0 ($\pm$ 0.2) & 80.8 ($\pm$ 0.3) & 1.05 \\
\bottomrule
\end{tabular}
\end{table}

\subsection{Backbone Generalization}
\label{app:backbone_generalization}

To assess whether the gains from VEM extend across different LM backbones, we report MIMIC-IV results on readmission, LOS, and drug recommendation for four backbones spanning both encoder-only models (BCMBERT and BioBERT) and decoder-only models (Meerkat and BioMistral). Table~\ref{tab:backbone_generalization_full} shows that VEM consistently improves performance over the corresponding LM-only baseline across all tasks and all reported metrics. The gains are especially pronounced for BCMBERT and Meerkat, where VEM improves readmission AUPRC by 5.0--5.8 points, LOS AUPRC by 5.8--6.2 points, and drug recommendation Jaccard by 3.4--4.6 points. Even for the stronger BioBERT and BioMistral baselines, VEM remains consistently beneficial; for example, it improves BioBERT readmission AUPRC from 45.4\% to 47.8\% and BioMistral drug recommendation F1 from 58.3\% to 62.4\%. Overall, these results indicate that Patients-like-me is not tied to a specific LM architecture and generalizes well across both encoder-only and decoder-only biomedical backbones.

\begin{table*}[t]
\centering
\caption{\textbf{Generalization of PLM across LM backbones on MIMIC-IV.} We evaluate PLM with the VEM algorithm across four backbones on readmission, LOS, and drug recommendation. VEM consistently improves performance over the corresponding LM-only baseline, showing that PLM generalizes across backbones.}
\label{tab:backbone_generalization_full}
\scriptsize
\resizebox{\textwidth}{!}{%
\begin{tabular}{llccccccc}
\toprule
\textbf{Backbone} & \textbf{Training} &
\multicolumn{2}{c}{\textbf{Readmission}} &
\multicolumn{2}{c}{\textbf{LOS}} &
\multicolumn{3}{c}{\textbf{Drug Recommendation}} \\
\cmidrule(lr){3-4}\cmidrule(lr){5-6}\cmidrule(lr){7-9}
& & \textbf{AUPRC} & \textbf{AUROC}
& \textbf{AUPRC} & \textbf{F1}
& \textbf{AUPRC} & \textbf{F1} & \textbf{Jaccard} \\
\midrule
\rowcolor{methodgreen}
BCMBERT-396M & VEM
& \textbf{48.6} & \textbf{80.6} & \textbf{82.8} & \textbf{35.0} & \textbf{75.4} & \textbf{64.8} & \textbf{48.9} \\
        & LM-only
& 43.6 & 74.5 & 77.0 & 30.7 & 70.2 & 61.4 & 45.5 \\
\midrule
\rowcolor{methodgreen}
Meerkat-8B & VEM
& \textbf{49.7} & \textbf{80.3} & \textbf{85.3} & \textbf{35.6} & \textbf{77.1} & \textbf{65.7} & \textbf{50.4} \\
           & LM-only
& 43.9 & 75.4 & 79.1 & 30.9 & 71.3 & 60.7 & 45.8 \\
\midrule
\rowcolor{methodgreen}
BioBERT-300M & VEM
& \textbf{47.8} & \textbf{80.1} & \textbf{80.5} & \textbf{33.8} & \textbf{75.1} & \textbf{64.2} & \textbf{48.5} \\
             & LM-only
& 45.4 & 77.4 & 78.1 & 33.7 & 73.1 & 61.7 & 46.8 \\
\midrule
\rowcolor{methodgreen}
BioMistral-7B & VEM
& \textbf{48.3} & \textbf{78.5} & \textbf{84.8} & \textbf{34.5} & \textbf{74.7} & \textbf{62.4} & \textbf{47.5} \\
               & LM-only
& 42.9 & 74.7 & 82.1 & 31.3 & 71.9 & 58.3 & 45.2 \\
\bottomrule
\end{tabular}%
}
\end{table*}

\subsection{Inductive Evaluation}
\label{app:inductive}

We further evaluate PLM in an inductive setting on MIMIC-IV with the Meerkat-8B backbone. The training patient graph serves as a fixed historical reference cohort. Each test patient is independently added to this graph, with edges constructed using the same diagnosis-overlap rule and threshold as in the main experiments. Inference is performed using the trained LM and GNN without parameter updates, and test patients are never connected to one another. As shown in Table~\ref{tab:inductive}, inductive PLM achieves performance close to the transductive setting and consistently outperforms the reported patient-independent baselines across all three tasks.

\begin{table}[t]
\centering
\caption{\textbf{Inductive evaluation on MIMIC-IV with Meerkat-8B.} We evaluate AUPRC (\%) across three clinical prediction tasks.}
\label{tab:inductive}
\small
\setlength{\tabcolsep}{4pt}
\renewcommand{\arraystretch}{1.1}
\begin{tabular}{lccc}
\toprule
\textbf{Setting} & \textbf{Readmission} & \textbf{LOS} & \textbf{Drug Rec.} \\
\midrule
PLM-Transductive & 49.7 & 85.3 & 77.1 \\
PLM-Inductive & 48.6 & 84.2 & 74.6 \\
Best independent baseline & 42.6 & 79.5 & 65.7 \\
\bottomrule
\end{tabular}
\end{table}

\subsection{Efficiency Analysis}
\label{subsec:efficiency_analy}

\begin{table}[t]
\centering
\caption{\textbf{Efficiency on MIMIC-IV with Meerkat-8B.} Time is hours per epoch, memory is peak GPU memory (GB), and metrics are AUPRC/F1/Jaccard (\%).}
\label{tab:efficiency_cost}
\footnotesize
\setlength{\tabcolsep}{3.5pt}
\renewcommand{\arraystretch}{1.12}
\begin{tabularx}{0.98\linewidth}{@{}lYYYYY@{}}
\toprule
\textbf{Training} & \textbf{Time (h)} & \textbf{Memory (GB)} & \textbf{AUPRC} & \textbf{F1} & \textbf{Jaccard} \\
\midrule
\rowcolor{methodgreen}
\textbf{VEM} & 2.55 & 80.4 & 77.1 & 65.7 & 50.4 \\
Alternating & 2.52 & 80.8 & 72.4 & 60.9 & 47.1 \\
E2E & 2.61 & 91.1 & 70.4 & 60.2 & 45.4 \\
2-stage & 2.63 & 80.8 & 73.2 & 63.6 & 47.9 \\
LM-only & 2.45 & 76.8 & 71.3 & 60.7 & 45.8 \\
\bottomrule
\end{tabularx}
\vspace{-0.4\baselineskip}
\end{table}

Table~\ref{tab:efficiency_cost} reports runtime and peak GPU memory on MIMIC-IV with the Meerkat-8B backbone. Relative to LM-only, VEM increases training time from 2.45 to 2.55 hours per epoch and peak memory from 76.8 to 80.4 GB, i.e., only a 4.1\% runtime increase and a 4.7\% memory increase. It also has lower runtime than E2E and 2-stage training, while maintaining comparable memory usage, showing that VEM delivers consistent performance gains with only modest computational overhead. Compared with alternating optimization, VEM requires nearly identical training time (2.55 vs.\ 2.52 hours per epoch) and slightly less GPU memory (80.4 vs.\ 80.8 GB), while improving AUPRC, F1, and Jaccard by 4.7, 4.8, and 3.3 \%, respectively.

\clearpage
\newpage

\end{document}